\documentclass{article}

\usepackage[final]{neurips_2023}
\usepackage[toc,page,header]{appendix}
\usepackage{minitoc}
\usepackage{paralist}

\usepackage{amssymb, amsmath, amsthm}
\usepackage{hyperref}       % hyperlinks
\usepackage{cleveref}
\usepackage{multirow}
\usepackage{booktabs} 

\usepackage{float}
\usepackage{caption}

\usepackage{geometry}
\usepackage{array}
\newcolumntype{C}[1]{p{#1}}
\usepackage{icomma}

\usepackage{wrapfig}
\newtheorem{theorem}{Theorem}
\newtheorem{example}[theorem]{Example}
\newtheorem{definition}[theorem]{Definition}
\newtheorem{lemma}[theorem]{Lemma}
\newtheorem{proposition}[theorem]{Proposition}
\newtheorem{corollary}{Corollary}[theorem]
\newtheorem*{claim*}{Claim}

\makeatletter
\newtheorem*{rep@theorem}{\rep@title}
\newcommand{\newreptheorem}[2]{%
	\newenvironment{rep#1}[1]{%
		\def\rep@title{#2 \ref{##1}}%
		\begin{rep@theorem}}%
		{\end{rep@theorem}}}
\makeatother

\newreptheorem{theorem}{Theorem}
\newreptheorem{proposition}{Proposition}
\newreptheorem{lemma}{Lemma}

\usepackage{mathtools}
\usepackage[ruled,linesnumbered,noend]{algorithm2e}

\usepackage{pgfplots}
\pgfplotsset{width=10cm,compat=1.9,tick label style={font=\Huge}}
\usetikzlibrary{arrows,backgrounds,automata,topaths,patterns,decorations.pathreplacing,decorations.pathmorphing}
\usetikzlibrary{arrows, chains, positioning, quotes, shapes.geometric}
\usetikzlibrary{calc}
\usetikzlibrary{fit}
\usetikzlibrary{arrows}
\usetikzlibrary{positioning,automata}
\usetikzlibrary{math}
\newcommand{\mulsetl}{\{\!\!\{}
\newcommand{\mulsetr}{\}\!\!\}}
\newcommand{\rsfoc}{\mathcal{RSFOC}_2}

\newcommand{\foc}{\mathcal{FOC}_2}
\newcommand{\feature}{\mathbf{x}}
\newcommand{\fset}{\textbf{X}}

\newcommand{\rgnn}{R-GNN\xspace}
\newcommand{\rrgnn}{R$^2$-GNN\xspace}

\newcommand{\relations}{P_2}

\newcommand{\ggraph}{\mathcal{G}_u}
\newcommand{\bgraph}{\mathcal{G}_b}
\newcommand{\sgraph}{\mathcal{G}_s}
\newcommand{\primal}{\emph{primal}}
\newcommand{\iaux}{\emph{aux1}}
\newcommand{\eaux}{\emph{aux2}}
\newcommand{\mg}{multi-relational graphs\xspace}

\usepackage[utf8]{inputenc} % allow utf-8 input
\usepackage[T1]{fontenc}    % use 8-bit T1 fonts
\usepackage{url}            % simple URL typesetting
\usepackage{booktabs}       % professional-quality tables
\usepackage{amsfonts}       % blackboard math symbols
\usepackage{nicefrac}       % compact symbols for 1/2, etc.
\usepackage{microtype}      % microtypography
\usepackage{xcolor}         % colors

\title{Calibrate and Boost Logical Expressiveness of GNN Over Multi-Relational and Temporal Graphs}

\author{%
  Yeyuan Chen\thanks{Equal contribution, listed in alphabetical order.}\\
  University of Michigan\\
  \texttt{yeyuanch@umich.edu}\\
  \And
  Dingmin Wang\footnotemark[1]\\
  University of Oxford\\
  \texttt{dingmin.wang@cs.ox.ac.uk}\\
}

\begin{document}
\doparttoc % Tell to minitoc to generate a toc for the parts
\faketableofcontents % Run a fake tableofcontents command for the partocs

\part{} % Start the document part

\maketitle

\begin{abstract}

As a powerful framework for graph representation learning, Graph Neural Networks (GNNs) have garnered significant attention in recent years. However, to the best of our knowledge, there has been no formal analysis of the logical expressiveness of GNNs as Boolean node classifiers over multi-relational graphs, where each edge carries a specific relation type. In this paper, we investigate $\foc$, a fragment of first-order logic with two variables and counting quantifiers. On the negative side, we demonstrate that the R$^2$-GNN architecture, which extends the local message passing GNN by incorporating global readout, fails to capture $\foc$ classifiers in the general case. Nevertheless, on the positive side, we establish that \rrgnn models are equivalent to $\foc$ classifiers under certain restricted yet reasonable scenarios. To address the limitations of \rrgnn regarding expressiveness, we propose a simple \emph{graph transformation} technique, akin to a preprocessing step, which can be executed in linear time. This transformation enables \rrgnn to effectively capture any $\foc$ classifiers when applied to the "transformed" input graph. Moreover, we extend our analysis of expressiveness and \emph{graph transformation} to temporal graphs, exploring several temporal GNN architectures and providing an expressiveness hierarchy for them. To validate our findings, we implement \rrgnn and the \emph{graph transformation} technique and conduct empirical tests in node classification tasks against various well-known GNN architectures that support multi-relational or temporal graphs. Our experimental results consistently demonstrate that \rrgnn with the graph transformation outperform the baseline methods on both synthetic and real-world datasets. The code is available at \url{https://github.com/hdmmblz/multi-graph}.

\end{abstract}

\section{Introduction}
Graph Neural Networks~(GNNs) have become a standard paradigm for learning with graph structured data, such as knowledge graphs~\cite{park2019estimating,cucala2021explainable,ourss} and molecules~\cite{hao2020asgn,gasteiger2021gemnet,guo2021few}. GNNs take as input a graph where each node is labelled by a feature vector, and then they recursively update the feature vector of each node by processing a subset of the feature vectors from the previous layer. For example, many GNNs update a node's feature vector by combining its value in the previous layer with the output of some aggregation function applied to its \emph{neighbours}' feature vectors in the previous layer; in this case, after $k$ iterations, a node's feature vector can capture structural information about the node’s $k$-hop neighborhood. GNNs have proved to be very efficient in many applications like knowledge graph completion and recommender systems. Most previous work on GNNs mainly revolves around finding GNN architectures (e.g.\ using different aggregation functions or graph-level pooling schemes) which offer good empirical performance~\cite{kipf2016semi,xu2018powerful,corso2020principal}. The theoretical properties of different architectures, however, are not yet well understood. 

In \cite{xu2018powerful}, the authors first proposed a theoretical framework to analyze the expressive power of GNNs by establishing a close connection between GNNs and the Weisfeiler-Lehman (1-WL) test for checking graph isomorphism. Similarly, \citet{geerts2022expressiveness} provides an elegant way to easily obtain bounds on the separation power of GNNs in terms of the Weisfeiler-Leman (k-WL) tests. However, the characterization in terms of the Weisfeiler-Lehman test only calibrates distinguishing ability. It cannot answer \emph{which Boolean node classifier can be expressed by GNNs}. To this end, 
\citet{Barcelo2020The} consider a class of GNNs named ACR-GNNs proposed in \cite{battaglia2018relational}, where the update function uses a ``global'' aggregation of the features of all nodes in the graph in addition to the typical aggregation of feature vectors of neighbour nodes. Then, the authors of the paper prove that in the single-relational~\footnote{The ``single-relational'' means there is only one type of edges in the graph. } scenario, ACR-GNNs can capture every Boolean node classifier expressible in the logic $\foc$.

However, most knowledge graphs need multiple relation types. For example, in a family tree, there are multiple different relation types such as "father" and "spouse". In this paper, we consider 
 the abstraction of a widely used GNN architecture called R-GCN~\cite{schlichtkrull2018modeling}, which is applicable to \mg. Following \cite{Barcelo2020The}, we define \rrgnn as a generalization of R-GCN by adding readout functions to the neighborhood aggregation scheme. We show that although adding readout functions enables GNNs to aggregate information of isolated nodes that can not be collected by the neighborhood-based aggregation mechanism, \rrgnn are still unable to capture all Boolean node classifiers expressible as formulas in logic $\foc$ in multi-relational scenarios if applied ``directly'' to the input. 
% As an example, we show two \mg~($G_1$ and $G_2$) in Figure~\ref{graphcompare}. We consider a Boolean node classifier used to classify as $true$ those nodes that are connected to a neighbour via two different relations~($p_1$ and $p_2$). The classifier can be expressed as a $\foc$ formula shown in Figure~\ref{graphcompare}. Clearly, it classifies node $a$ as $true$ in $G_1$ and as $false$ in $G_2$. However, we can easily observe that \rrgnn will assign the same value to $a$ in both graphs, as they are unable to distinguish whether the $p_1$-neighbour of $a$ is the same as its $p_2$-neighbour. 
%
This leaves us with the following questions:
(1) Are there reasonable and practical sub-classes of \mg for which $\foc$ can be captured by \rrgnn? (2) Is there some simple way to encode input graphs, so that all $\foc$ node classifiers can be captured by \rrgnn for all \mg?

In this paper, we provide answers to the above questions. Moreover, we show that our theoretical findings also transfer to temporal knowledge graphs, which are studied extensively in \cite{park2022evokg} and \cite{gao2022equivalence}. In particular, we leverage the findings from \cite{gao2022equivalence} which shows that a temporal graph can be transformed into an ``equivalent'' static multi-relational graph.
Consequently, our results, originally formulated for static multi-relational graphs, naturally extend to the domain of temporal knowledge graphs. Our contributions are as follows:
\begin{compactitem}
    \item We calibrate the logic expressiveness of \rrgnn as node classifiers over different sub-classes of multi-relational graphs.
    \item In light of some negative results about the expressiveness of \rrgnn found in the multi-relational scenario, there is a compelling need to boost the power of \rrgnn. To address this challenge, we propose a \emph{ graph transformation} and show that such a transformation enables R$^2$-GNN to capture each classifier expressible as a $\foc$ formula in all multi-relational graphs.
    %
    % is applied to the input graph, then can be captured by a R$^2$-GNN. 
    \item We expand the scope of expressiveness results and graph transformation from static \mg to \emph{temporal} settings. Within this context, we propose several temporal GNN architectures and subject them to a comparative analysis with frameworks outlined in \cite{gao2022equivalence}. Ultimately, we derive an expressiveness hierarchy.
\end{compactitem}

\section{Preliminaries}
\subsection{Multi-relational Graphs}\label{multigraph}
A \emph{multi-relational} graph is a $4$-tuple  
$G = (V, \mathcal{E}, P_1, P_2)$, where $V$, $P_1$, $P_2$ are finite sets of \emph{nodes}, \emph{types}
and \emph{relations}~(a.k.a, unary/binary predicates)\footnote{\noindent For directed graphs, we assume $P_2$ contains relations both in two directions (with inverse-predicates). Moreover, we assume there exists an "equality relation" $\mathsf{EQ}\in P_2$ such that $\forall x,y\in V, x=y\Leftrightarrow \mathsf{EQ}(x,y)=1$.}, respectively, 
%
% is a finite set of \emph{nodes},  $P_1$ is a finite set of \emph{types}~(a.k.a unary predicates)\footnote{Follow~\cite{schlichtkrull2018modeling},  for directed graphs, we assume $P_2$ contains relations both in two directions.}, $P_2$ is a finite set of \emph{relations}~(a.k.a binary predicates)
and $\mathcal{E}$ is a set of triples of the form $(v_1, p_2, v_2)$ or $(v, type, p_1)$, where $p_1 \in P_1$, $p_2 \in P_2$, $v_1, v_2, v \in V$, and $type$ is a special symbol. 

Next, given arbitrary~(but fixed) finite sets $P_1$ and $P_2$ of unary and binary predicates, respectively, we define the following three kinds of graph classes:
\begin{compactitem}
    \item a \emph{universal} graph class can be any set of graphs of the form $(V, \mathcal{E}, P_1, P_2)$.
    \item a \emph{bounded} graph class is a universal graph class for which there exists $n \in \mathbb{N}$ such that each graph in the class has no more than $n$ nodes;
    \item a \emph{simple} graph class is a universal graph class where for each graph $(V, \mathcal{E}, P_1, P_2)$ in the class, and for each pair of nodes $v_1, v_2 \in V$, there exists at most one 
    triple in $\mathcal{E}$ of the form $(v_1,p_2,v_2)$, where $p_2 \in P_2$.  
\end{compactitem}
% \textcolor{red}{Yeyuan: All graphs in a graph class must share the same $P_1$, $P_2$, but it isn't illustrated in definition above, The reason is that: we can't define logical classifier on graph class that consists of graphs from different predicate set.}
We typically use symbols $\ggraph$, $\bgraph$, and $\sgraph$ to denote universal, bounded, and simple graph classes, respectively. 

\begin{definition}
For a given graph class over predicates $P_1$ and $P_2$, 
a \emph{Boolean node classifier} is a function $\mathcal{C}$ such that for each graph $G=(V, \mathcal{E}, P_1, P_2)$ in that graph class, and each $v \in V$, $\mathcal{C}$ classifies $v$ as $true$ or $false$. 
\end{definition}

\subsection{Graph Neural Networks}
\paragraph{Node Encoding}\label{nodeenc} We leverage a GNN as a Boolean node classifier for multi-relational graphs, which cannot be directly processed by GNN architectures, requiring graphs where each node is labelled by an initial feature vector. 
Therefore, we require some form of \emph{encoding} to map a multi-relational graph to a suitable
input for a GNN. Such an encoding should keep graph permutation invariance ~\cite{geerts2021expressiveness} since we don't want a GNN to have different outputs for isomorphic graphs. 
 Inspired by~\cite{liu2021indigo}
 %,cucala2021explainable}, 
 for a multi-relational graph $G=(V, \mathcal{E}, P_1,P_2)$ and an ordering
 $p_1, p_2, \cdots, p_k$ of the predicates in $P_1$,
 we define an initialization function $I(\cdot)$ which maps each node $v \in V$ to a Boolean feature vector $I(v)=\mathbf{x}_v$ with a fixed dimension $|P_1|$, where the $i$th component of the vector is set to 1 if and only if the node $v$ is of the type $p_i$, that is, $(\mathbf{x}_v)_i =1$ if and only if $(v, type, p_i) \in \mathcal{E}$. If $P_1$ is an empty set, we specify that each node has a  1-dimension feature vector whose value is $1$. Clearly, this encoding is permutation invariant. 

\paragraph{\rgnn}
R-GCN \cite{schlichtkrull2018modeling} is a widely-used GNN architecture that can be applied to multi-relational graphs. By allowing different aggregation and combination functions, we extend R-GCN to a  more general form which we call R-GNN. Formally, let $\left\{\{A_{j}^{(i)}\}_{j=1}^{|P_2|}\right\}_{i=1}^{L}$ and $\{C^{(i)}\}_{i=1}^{L}$
 be two sets of \emph{aggregation} and \emph{combination} functions. An R-GNN computes vectors $\mathbf{ x}^{(i)}_v$ for every node $v$  of the \emph{multi-relational} graph $G=(V, \mathcal{E}, P_1,P_2)$ on each layer $i$, via the recursive formula
\begin{equation}\label{acgnn}
    \textbf{x}_v^{(i)}=C^{(i)}\left( \textbf{x}_v^{(i-1)},\left(A_j^{(i)}(\mulsetl \textbf{x}_u^{(i-1)}|u\in \mathcal{N}_{G,j}(v)\mulsetr)\right)_{j=1}^{|\relations|}\right)
\end{equation}
where $x_v^{(0)}$ is the initial feature vector as encoded by $I(\cdot)$, $\mulsetl \cdot \mulsetr$ denotes a \emph{multiset},  $( \cdot )_{j=1}^{|P_2|}$ denotes a tuple of size $|P_2|$, $\mathcal{N}_{G,j}(v)$ denotes the neighbours of $v$ via a binary relation $p_j \in \relations$, that is, nodes $w \in V$ such that $(v,p_j,w) \in \mathcal{E}$.
% \footnote{Here and in the rest of the paper we omit bias terms for clarity.}

\noindent 
\paragraph{\rrgnn}
\rrgnn  extends \rgnn by specifying readout functions $\{R^{(i)}\}_{i=1}^{L}$
, which aggregates the feature vectors of all the nodes in a graph. The vector $\mathbf{x}^{(i)}_v$ of each node $v$ in $G$ on each layer $i$, is computed by the following formula

\begin{equation}\label{acrgnn}
    \textbf{x}_v^{(i)}=C^{(i)}\left( \textbf{x}_v^{(i-1)},\left(A_j^{(i)}(\mulsetl \textbf{x}_u^{(i-1)}|u\in \mathcal{N}_{G,j}(v)\mulsetr)\right)_{j=1}^{|\relations|}, R^{(i)}(\mulsetl \textbf{x}_u^{(i-1)}|u\in V\mulsetr)\right) 
\end{equation}

Every layer in an R$^2$-GNN first computes the aggregation over all the nodes in $G$; then, for every node $v$, it computes the aggregation over the neighbors of $v$; and finally, it combines the features of $v$ with the two aggregation vectors; the result of this operation is the new feature 
vector for $v$. Please note that an R-GNN can be seen as a special type of R$^2$-GNN where the combination function simply ignores the output of the readout function.

It is worth noting that R-GNN as well as R$^2$-GNN is not a specific model architecture; it is a framework that contains a bunch of different GNN architectures. In the paper, we mentioned it's generalized from R-GCN (\cite{schlichtkrull2018modeling}), but our primary objective is to establish a comprehensive framework that serves as an abstraction of most Message-Passing GNNs (MPGNN). In the definitions (\Cref{acgnn,acrgnn}), the functions can be set as any functions, such as matrix multiplications or QKV-attentions. Most commonly used GNN such as R-GCN (\cite{schlichtkrull2018modeling}) and R-GAT (\cite{busbridge2019relational}) are captured (upper-bounded) within our R-GNN frameworks. Other related works, such as (\cite{Barcelo2020The,huang2023theory,qiu2023logical}) also use intrinsically the same framework as our R-GNN/R$^2$-GNN, which has been widely adopted and studied within the GNN community. We believe that analyzing these frameworks can yield common insights applicable to numerous existing GNNs

\paragraph{GNN-based Boolean node classifier} In order to translate the output of a GNN to a Boolean value, we apply a Boolean classification function $CLS: \mathbb{R}^d \to \{true,false\}$,
where $d$ is the dimension of the feature vectors $\textbf{x}_v^L$. Hence, a Boolean node classifier based on an R$^2$-GNN $\mathcal{M}$ proceeds
in three steps: (1) encode the input multi-relational graph $G$ as described above,
(2) apply the R$^2$-GNN, and (3) apply $CLS$ to the output of the  R$^2$-GNN.
This produces a $true$ or $false$ value for each node of $G$. In what follows, we abuse the language and represent a family of GNN-based Boolean node classifiers by the name of the corresponding GNN architecture; for example, R$^2$-GNN is the set of all R$^2$-GNN-based Boolean node classifiers.

%Hence, a R$^2$-GNN-based Boolean node classifier with $L$ layers is defined as a tuple $(\{A^{(i)}\}_{i=1}^{L}, \{C^{(i)}\}_{i=1}^{L}, \{R^{(i)}\}_{i=1}^{L}, CLS)$.

\subsection{Logic $\foc$ Formulas}
In this paper, we focus on the logic $\foc$, a fragment of first-order logic that only allows formulas with at most two variables, but in turn permits to use \emph{counting quantifiers}. Formally, given two finite sets $P_1$ and $P_2$ of \emph{unary} and \emph{binary} predicates, respectively, a $\foc$  formula $\varphi$ is inductively defined according to the following grammar:
\begin{equation}\label{focdef1}
\varphi ::= A(x) \mid r(x,y) \mid \varphi\wedge\varphi \mid  \varphi\vee\varphi \mid  \neg\varphi \mid \exists^{\ge n}y(\varphi) \text{ where } A\in P_1 \text{ and } r\in P_2
\end{equation}

where $x/y$ in the above rules can be replaced by one another. But please note that $x$ and $y$ are the only variable names we are allowed to use (Though we can reuse these two names). In particular, a $\foc$ formula $\varphi$ with exactly one free variable $x$ represents a Boolean node classifier for multi-relational graphs as follows: a node $v$ is assigned to $true$ iff the formula $\varphi_v$ obtained by substituting $x$ by $v$ is satisfied by the (logical) model represented by the multi-relational graph. Similarly as the GNN-based Boolean node classifiers, in what follows, we abuse the language and represent the family of $\foc$ Boolean node classifiers by its name $\foc$. 
\subsection{Inclusion and Equality Relationships}
In this paper, we will mainly talk about inclusion/non-inclusion/equality/strict-inclusion relationships between different node classifier families on certain graph classes. To avoid ambiguity, we give formal definitions of these relationships here. These definitions are all quite natural.
\begin{definition}
For any two sets of node classifier $A,B$, and graph class $\mathcal{G}$, We say:
\begin{compactitem}
\item $A\subseteq B$ on $\mathcal{G}$, iff for any node classifier $a\in A$, there exists some node classifier $b\in B$ such that for all graph $G\in\mathcal{G}$ and $v\in V(G)$, it satisfies $a(G,v)=b(G,v)$ (Namely, $a$ and $b$ evaluate the same for all instances in $\mathcal{G}$). It implies $B$ is more expressive than $A$ on $\mathcal{G}$.
\item $A\nsubseteq B$ on $\mathcal{G}$, iff the above condition in item 1 doesn't hold.
\item $A\subsetneq B$ on $\mathcal{G}$, iff $A\subseteq B$ but $B\nsubseteq A$. It implies $B$ is \textbf{strictly} more expressive than $A$ on $\mathcal{G}$.
\item $A=B$ on $\mathcal{G}$, iff $A\subseteq B$ and $B\subseteq A$. It implies $A$ and $B$ has the same expressivity on $\mathcal{G}$.
\end{compactitem}
\end{definition}
\section{Related Work}
The relationship between first-order logic and the Weisfeiler-Lehman test was initially established by \cite{63543}. Subsequently, more recent works such as \cite{xu2018powerful}, have connected the Weisfeiler-Lehman test with expressivity of GNN. 
This line of research has been followed by numerous studies, including \cite{maron2020provably}, which explore the distinguishability of GNNs using the Weisfeiler-Lehman test technique. In particular, 
\cite{Barcelo2020The} introduced the calibration of logical expressivity in GNN-based classifiers and proposed a connection between $\foc$ and R$^2$-GNN in single-relational scenario. This led to the emergence of related works, such as \cite{huang2023theory}, \cite{geerts2021expressiveness}, and \cite{qiu2023logical}, all of which delve into the logical expressivity of GNNs. Moreover, the theoretical analysis provided in \cite{gao2022equivalence} has inspired us to extend our results to temporal graph scenarios.
%
%
% \cite{Barcelo2020The} first calibrate logical expressivity of GNN-based classifier and propose connection between $foc$ and R$^2$-GNN in single-relational scenario. Since then there are many related work exploring logical expressivity of GNN such as \cite{huang2023theory,geerts2021expressiveness} and \cite{qiu2023logical}. Furthermore, Theoretical analysis in \cite{gao2022equivalence} inspires us to transfer our results to temporal graph scenarios.
\section{Logic expressiveness of \rrgnn in \mg}
\begin{figure}[t!]
  \vspace{-0.8cm}
  \centering
  \setlength\tabcolsep{10pt}
  \renewcommand{\arraystretch}{1.4}
  \begin{tabular}{cc}
  \begin{tikzpicture}[scale=0.75]
        \node at (-3, -1) {$G_1:$}; 
        \node[circle, dotted, draw, inner sep=2pt,minimum size=14pt] (a) at (-2,-2) {a};
        % \node[circle,draw=black,inner sep=2pt,minimum size=12pt,label=left:{\footnotesize a}] (a) at (-2,-2) {a};
  
        \node[circle,draw=black,inner sep=2pt,minimum size=12pt] (b) at (0,0) {b};
        %fill=blue!50
        \draw[line width=0.5mm,black, -] (a) to[bend right] node[midway,below,sloped] {$p_1$} (b);
        \draw[line width=0.5mm,black, -] (a) to[bend left] node[midway,above,sloped] {$p_2$} (b);
        
        \node[circle,draw=black,inner sep=2pt,minimum size=12pt] (d) at (3,0) {d};
        
        %fill=blue!50
        \node[circle,draw=black,inner sep=2pt,minimum size=12pt] (c) at (1,-2) {c};
        
        \draw[line width=0.5mm,black, -] (c) to[bend right] node[midway,below,sloped] {$p_1$}  (d);
        \draw[line width=0.5mm,black, -] (c) to[bend left] node[midway,above,sloped] {$p_2$} (d);
        
  \end{tikzpicture}
  
  &
   \begin{tikzpicture}[scale=0.8]
        \node at (-2.5, -1) {$G_2:$}; 
        \node[circle, dotted, draw, inner sep=2pt,minimum size=12pt] (a) at (-2,-2) {a};
        \node[circle,draw=black,inner sep=2pt,minimum size=12pt] (b) at (0,0) {b};
        
                %fill=blue!50
        \node[circle,draw=black,inner sep=2pt,minimum size=12pt] (c) at (1,-2) {c};

        \node[circle,draw=black,inner sep=2pt,minimum size=12pt] (d) at (3,0) {d};
                
        %fill=blue!50
        \draw[line width=0.5mm,black, -] (a) -- node[midway, sloped, above] {$p_1$} (b);
        
        \draw[line width=0.5mm,black, -] (a) -- node[below]{$p_2$} (c);

        \draw[line width=0.5mm,black, -] (c) -- node[midway,below,sloped] {$p_1$}  (d);
        
        \draw[line width=0.5mm,black, -] (b) -- node[above] {$p_2$} (d);
        
  \end{tikzpicture}
  \\
  \multicolumn{2}{c}{Node classifier: $\varphi(x):=\exists^{\ge 1}y(p_1(x,y)\wedge p_2(x,y))$.}
  \end{tabular}
  \caption{Multi-edge graphs $G_1$ and $G_2$, and a $\foc$ formula $\varphi(x)$ that distinguishes them; $\varphi(x)$ evaluates node $a$ in $G_1$ to $true$ and node $a$ in $G_2$ to $false$. 
%   We assume that all nodes have the same one-hot initial feature vector. 
  %However, an ACR-GNN can not distinguish the them.
  \label{graphcompare}}
  %\vspace{-0.5cm}
\end{figure}
Our analysis begins with the observation that certain Boolean classifiers can be represented as $\foc$ formulas, but remain beyond the expressiveness of any R$^2$-GNN (and consequently, any R-GNN or R-GCN). An illustrative example of this distinction is provided in Figure~\ref{graphcompare}. In this example, we make the assumption that $P_1$ is empty, thereby ensuring that all nodes in both $G_1$ and $G_2$ possess identical initial feature vectors. Additionally, $P_2$ is defined to comprise precisely two relations, namely, $p_1$ and $p_2$. It is evident that no R$^2$-GNN can distinguish the node $a$ in $G_1$ from node $a$ in $G_2$ -- that is, when an R$^2$-GNN performs the neighbour-based aggregation, it cannot distinguish whether the $p_1$-neighbour of $a$ and the $p_2$-neighbour of $a$ are the same. Moreover, the global readout aggregation cannot help in distinguishing those nodes because all nodes have the same feature vector.
% In fact, it is easy to show, using induction, that the feature vectors that an arbitrary R$^2$-GNN will assign, in each layer, to nodes $a$, $b$, $c$, and $d$ when applied to $G_1$ will be the same as those it will assign to the same nodes when applied to $G_2$.

We proceed to formalize this intuition and, in the reverse direction, offer a corresponding result. We demonstrate that there exist Boolean classifiers that fall within the scope of \rrgnn but elude capture by any $\foc$ formula.

% We now formalise this intuition and furthermore provide an analogous result in the opposite direction, showing that there are Boolean classifiers captured by \rrgnn which cannot be captured by any $\foc$ formula. 
\begin{proposition}\label{inclusion}
$\foc \not\subseteq $ \rrgnn and \rrgnn $\not\subseteq \foc$  on some universal graph class.
\end{proposition}

We prove \Cref{inclusion} in the Appendix. Here, we give some intuition about the proof. The first result is proved using the example shown in Figure~\ref{graphcompare}, which we have already discussed. To show \rrgnn $\not\subseteq \foc$, we construct a classifier $c$ which classifies a node into true iff \emph{the node has a larger number of $r_1$-type neighbors than that of $r_2$-type neighbors.} We can prove that we can easily construct an R$^2$-GNN to capture $c$. However, for $\foc$, this cannot be done, since we can only use counting quantifiers expressing that there exist at most or at least a specific number of neighbours connected via a particular relation, but our target classifier requires comparing 
 indefinite numbers of neighbours via two relations. Thus, we proceed by contradiction,
assume that there exists a $\foc$ classifier equivalent to $c$,
and then find two large enough graphs with nodes
that cannot be distinguished by the classifier (but can be distinguished by $c$).

In some real-world applications, it is often possible to find an upper bound on the size of any possible input graph or to ensure that any input graph will contain at most one relation between every two nodes. For this reason, we next present 
 restricted but positive\&practical expressiveness results on bounded and simple graph classes.

\begin{theorem}\label{single-edge}
$\foc \subseteq$ \rrgnn on any simple graph class, and $\foc\subsetneq$ \rrgnn on some simple graph class.
\end{theorem}
The key idea of the construction is that we will first transform the $\foc$ formula into a new form which we call \emph{relation-specified} $\foc$ ~(an equivalent form to $\foc$, see more details in our Appendix), and then we are able to construct an equivalent R$^2$-GNN inductively over the parser tree of the transformed formula. 

Having \Cref{single-edge}, one may wonder about the inclusion relationship of \rrgnn and $\foc$ in the backward direction. In \Cref{inclusion}, we showed that for arbitrary universal graph classes, this inclusion relationship fails. However, given a bounded graph class, we can show that for each R$^2$-GNN Boolean node classifier, one can write an equivalent $\foc$ classifier. An intuition about why this is the case is that all graphs in a bounded graph class will have at most $n$ constants, for some known $n \in \mathbb{N}$, so for each R$^2$-GNN classifier, we can construct an equivalent $\foc$ classifier with a finite number of sub-formulas to recover the features obtained at different layers of R$^2$-GNN.

\begin{theorem}\label{ACR-GNN}
\rrgnn $\subseteq$ $\foc$ on any bounded graph class, and \rrgnn $\subsetneq$ $\foc$ on some bounded graph class.
\end{theorem}

Combining \Cref{single-edge} and \Cref{ACR-GNN}, we have the following corollary.
\begin{corollary}\label{cor9}
\rrgnn $=\foc$ on any \emph{bounded} \emph{simple} graph class. 
\end{corollary}

At Last, one may be curious about the complexity of logical classifier in \Cref{ACR-GNN}. Here we can give a rather loose bound as follows:

\begin{theorem}\label{complexity}
For any bounded graph class $\bgraph$. Suppose any $G\in\bgraph$ has no more than $N$ nodes, and $\bgraph$ has unary predicate set $P_1$ and relation (binary predicate) set $P_2$. Let $m_1:=|P_1|,m_2:=|P_2|$, then for any node classifier $c$, suppose $c$ can be represented as an \rrgnn with depth (layer number) $L$, then by \Cref{ACR-GNN} there is a $\foc$ classifier $\varphi$ equivalent to $c$ over $\bgraph$, and the following hold:
\begin{compactitem}
\item The quantifier depth of $\varphi$ is no more than $L$.

\item The size of $\varphi$ (quantified by the number of nodes of $\varphi$'s parse tree) is no more than $2^{2f(L)}$, where $f(L):=2^{2^{2(N+1)f(L-1)}},f(0)=O(2^{2^{2(m_1+m_2)}})$.
\end{compactitem}
\end{theorem}

The key idea of \Cref{complexity} is the following: First, by \Cref{truthtable} in our appendix, the combination of \textbf{ALL} $\foc$ logical classifiers with quantifier depth no more than $L$ can already distinguish accepting and rejecting instances of $c$.
 Then by \Cref{finite} (This is a key point of this bound; please refer to our appendix), We know the number of intrinsically different bounded-depth 
 $\foc$ classifiers is finite, so we only need to get an upper bound on this number. Finally, we can get the desired bound by iteratively using the fact that a boolean combination of a set of formulas can be always written as DNF (disjunctive normal form). The tower of power of two comes from $L$ rounds of DNF enumerations. Although the bound seems scary, it is a rather loose bound. We give a detailed proof of \Cref{complexity} in the appendix along with the proof of \Cref{ACR-GNN}.
\section{\rrgnn capture $\foc$ over transformed \mg }\label{sec:acrgnntg}
As we pointed out in the previous section, one of the reasons why \rrgnn cannot capture $\foc$ classifiers over arbitrary universal graph classes is that in \mg, they cannot distinguish whether information about having a neighbour connected via a particular relation comes from the same neighbour node or different neighbour nodes. Towards solving this problem, we propose a \emph{graph transformation} $F$~(see \Cref{trans_def}), which enables \rrgnn to capture all $\foc$ classifiers on \mg. Similar transformation operations have also been used and proved to be an effective way to encode \mg in previous studies, e.g., MGNNs~\cite{cucala2021explainable}, Indigo~\cite{liu2021indigo} and Time-then-Graph~\cite{gao2022equivalence}. 
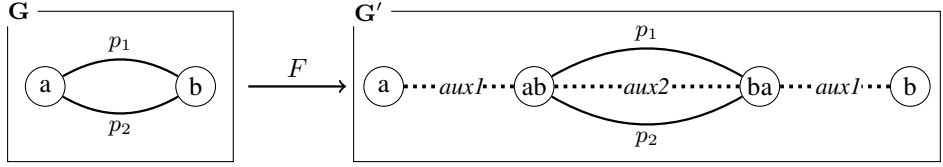
\begin{figure}[t!]
  \centering

  \begin{tikzpicture}

        \draw[draw=black] (-0.5,0.8) -- (-0.5,-1) -- (2.5, -1) -- (2.5,1) -- (-0.1, 1);
        \node at (-0.35, 1) {\small $\mathbf{G}$};
        \node[circle, draw=black, inner sep=0pt,minimum size=15pt] (a) at (0,0) {a};
        \node[circle, draw=black, inner sep=0pt,minimum size=15pt] (b) at (2, 0) {b};
        \draw[line width=0.3mm,black, -] (a) to[bend left] node[midway,yshift=0.2cm]  {\small $p_1$}   (b);
        \draw[line width=0.3mm,black, -] (a) to[bend right] node[midway,yshift=-0.2cm]  {\small $p_2$}  (b);

        \draw[draw=black] (4.1,0.8) -- (4.1,-1) -- (11.9, -1) -- (11.9,1) -- (4.6, 1);
        \node at (4.3, 1) {\small $\mathbf{G'}$};

        \draw[->, line width=0.03cm] (2.7,0) to node[above,midway] {$F$} (4,0); 
        
        %\draw[draw=black] (4.2,1) rectangle ++(7.6,-2);
        \node[circle, draw=black, inner sep=0pt,minimum size=15pt] (a) at (4.5,0) {a};
        
        \node[circle, draw=black, inner sep=0pt,minimum size=15pt] (ab) at (6.5,0) {ab};
        
        %fill=blue!50
        \node[circle, draw=black, inner sep=0pt,minimum size=15pt] (b) at (11.5,0) {b};
        
        \node[circle, draw=black, inner sep=0pt,minimum size=15pt] (ba) at (9.5,0) {ba};

        %  \draw[line width=0.3mm,black, -] (ab) to node[midway,yshift=0.2cm]  {\small $e_1^3$}  (ba);
        
         \draw[line width=0.3mm,black, -] (ab) to[bend left] node[midway,yshift=0.2cm]  {\small $p_1$}  (ba);
         
         \draw[line width=0.3mm,black, -] (ab) to[bend right] node[midway,yshift=-0.2cm]  {\small $p_2$}  (ba);
         
        %  \draw[line width=0.3mm,black, -] (ab) to node[midway, above]  {\small \eaux}  (ba);
         
         \draw[dotted,line width=0.05cm] (ab) -- (7.7,0);
         \draw[dotted,line width=0.05cm] (ba) -- (8.3, 0);
         \node at (8,0) {\small \eaux};

        % \draw[line width=0.3mm,black, -] (a) to node[midway,above]  {\small \iaux}  (ab);
        
         \draw[dotted,line width=0.05cm] (a) -- (5.25,0);
         \draw[dotted,line width=0.05cm] (ab) -- (5.8, 0);
         \node at (5.55,0) {\small \iaux};

        %\draw[line width=0.3mm,black, -] (b) to node[midway,above]  {\small \iaux}   (ba);
        
         \draw[dotted,line width=0.05cm] (ba) -- (10.25,0);
         \draw[dotted,line width=0.05cm] (b) -- (10.85, 0);
         \node at (10.55,0) {\small \iaux};

        %\node at (-1.5, -0.6) {};
        
    \end{tikzpicture} 
  \caption{Graph Transformation.} \label{fig:single}
\end{figure}
\begin{definition}\label{trans_def}
Given a multigraph $G=(V, \mathcal{E}, P_1, P_2)$, the transformation $F$ will map $G$ to another graph $F(G)=(V', \mathcal{E}', P'_1, P'_2)$ with changes described as follows: 
\begin{compactitem}
    \item for any two nodes $a, b\in V$, if there exists at least one relation $p\in P_2$ between $a$ and $b$, we add two new nodes $\text{ab}$ and $\text{ba}$ to $V'$.
    \item  we add a new unary predicate \{\primal\} and two new binary predicates \{\iaux,\eaux\}. Hence,  $F(P_1):=P'_1 = P_1 \cup \{\primal\}$, and $F(P_2):=P'_2 = P_2 \cup  \{ \iaux, \eaux\}$. For each node $v' \in V'$, $\primal(v')=1$ iff $v'$ is also in  $V$; otherwise, $\primal(v')=0$;    
    \item for each triplet of the form $(a,p_2, b)$ in $\mathcal{E}$, we add to $\mathcal{E}'$ four new triples: $(\text{ab},\iaux, a)$, $(\text{ba},\iaux, b)$ and $(\text{ab},\eaux, \text{ba})$ as well as  $(\text{ab},p_2, \text{ba}) $. 
\end{compactitem}
\end{definition}

An example is in  Figure~\ref{fig:single}. We can see that after applying the \emph{graph transformation}, we need to execute two more hops to propagate information from node $a$ to node $b$. However, now we
are able to distinguish whether the information about different relations comes from the same node or different nodes. This transformation can be implemented and stored in linear time/space complexity $O(|V|+|\mathcal{E}|)$, which is very efficient.

\begin{wrapfigure}{r}{0.3\textwidth}
\centering
%\resizebox {0.35\textwidth}{2.2cm} {
      \begin{tikzpicture}
         \draw (0,0) ellipse (1.9cm and 1.2cm) node at (0,0.8) {\small  \rrgnn$\circ F$};
        \draw (-0.7,-0.2) ellipse  (1cm and 0.5cm) node at (-0.8,-0.2) {\small \rrgnn};
        \draw (0.8,-0.2) ellipse  (0.8cm and 0.5cm) node at (0.86,-0.2) {\small $\foc$};
   \end{tikzpicture}
%}
    \caption{\small Relations of
    \rrgnn, $\foc$ and \rrgnn $\circ F$.}
    \label{fig:relations}

\end{wrapfigure}
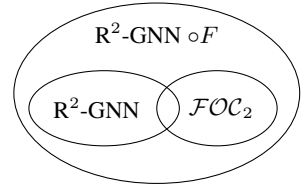

\begin{definition}Given a classifier $\mathcal{C}$ and a transformation function $F$, we define $\mathcal{C}\circ F$ to be a new classifier,  an extension of $\mathcal{C}$ with an additional transformation operation on the input graph.
\end{definition}
With \emph{graph transformation} $F$, we get a more powerful class of classifiers than R$^2$-GNN. We analyze the logical expressiveness of \rrgnn  $\circ F$ in \mg, which means first transform a graph $G$ to $F(G)$ and then run an R$^2$-GNN on $F(G)$. We will see in the following that this transformation $F$ boosts the logical expressiveness of \rrgnn prominently. 

\begin{theorem}\label{theoremacracrf}
\rrgnn $\subseteq$  \rrgnn $\circ F$ on any \emph{universal} graph class.
\end{theorem}

\begin{theorem}\label{theoremfocacrf}
$\foc \subseteq$ \rrgnn $\circ F$ on any  \emph{universal} graph class.
\end{theorem}
\Cref{theoremacracrf} demonstrates that \rrgnn with \emph{graph transformation} $F$ have more  expressiveness than \rrgnn; and \Cref{theoremfocacrf} shows the connection between $\foc$ and  \rrgnn equipped with \emph{graph transformation} $F$.  We depict their relations in Figure~\ref{fig:relations}. \Cref{theoremacracrf} is a natural result since no information is lost in the process of transformation, while \Cref{theoremfocacrf} is an extension on \Cref{single-edge}, whose formal proofs can be found in the Appendix. As for the backward direction, we have the result shown in \Cref{theoremacrffoc}.
\begin{theorem}\label{theoremacrffoc}
 \rrgnn $\circ F \subseteq \foc$ on  any \emph{bounded} graph class.
\end{theorem}
The proof of the theorem is relatively straightforward based on previous results: by \Cref{ACR-GNN}, it follows that \rrgnn $\circ F\subseteq\foc\circ F$ on any bounded graph class. Then, it suffices to prove $\foc\circ F\subseteq \foc$, which we do by using induction over the quantifier depth.

By combining \Cref{theoremfocacrf} and \Cref{theoremacrffoc}, we obtain \Cref{boundedgnnfoc}, stating that $\foc$ and \rrgnn $\circ F$ have the same expressiveness with respect to bounded graph classes.  \Cref{boundedgnnfoc} does not hold for arbitrary universal graph classes, but 
our finding is nevertheless exciting because, in many real-world applications there are upper bounds over input graph size.

\begin{corollary} \label{boundedgnnfoc}
\rrgnn  $\circ F = \foc$  on any bounded graph class.
\end{corollary}

To show the strict separation as in \Cref{fig:relations}, we can combine \Cref{inclusion,single-edge,theoremacracrf} and \Cref{theoremfocacrf} to directly get the following:
\begin{corollary}\label{strictsep}
\rrgnn $\subsetneq$ \rrgnn$\circ F$ on some universal graph class, and $\foc\subsetneq$ \rrgnn on some simple graph class.
\end{corollary}
One may think after transformation $F$, the logic $\foc\circ F$ with new predicateds becomes stronger as well. However by a similar proof as for \Cref{theoremfocacrf} and \Cref{enum}, we can actually show $\foc\circ F\subseteq \foc$ always holds, so $F$ won't bring added power for $\foc$. However, it indeed make \rrgnn strictly more expressive.
\section{Temporal Graphs}
% In real-world scenarios, most relations or facts exist only in a limited time range. For example, in Facebook's social media network, it is very common that the relationship between two persons might change or disappear after a certain amount of time. Hence, many real-world graphs are temporal in nature. 

As stated in \cite{gao2022equivalence}, a temporal knowledge graph, composed of multiple snapshots, can consistently undergo transformation into an equivalent static representation as a multi-relational graph. Consequently, this signifies that our theoretical results initially devised for \mg can be extended to apply to temporal graphs, albeit through a certain manner of transfer.

% According to \cite{gao2022equivalence}, a temporal knowledge graph containing several snapshots can always be transformed into an equivalent static form represented as a multi-relational graph. This means that our theoretical results for \mg can be transferred to temporal graphs in some way.

Following previous work~\cite{jin2019recurrent,pareja2020evolvegcn,park2022evokg,gao2022equivalence}, we define a temporal knowledge graph as a set of graph ``snapshots'' distributed over a sequence of \textbf{finite} and  \textbf{discrete} time points $\{1, 2, \dots, T\}$.
Formally, a temporal knowledge graph is a set $G = \{G_1, \cdots, G_T\}$ for some $T \in \mathbb{N}$, where each $G_t$ is a static multi-relational graph. All these $G_t$ share the same node set and predicate set.

In a temporal knowledge graph, a relation or unary fact between two nodes might hold or disappear across the given timestamps. For example, a node $a$ may be connected to a node $b$ via a relation $p$ in the first snapshot, but not in the second; in this case, we have $(a, p, b)$ in $G_1$ not in $G_2$.
% To represent a temporal knowledge graph as a static knowledge graph, we will combine
% all snapshots into a single graph.
To keep track of which relations hold
at which snapshots, we propose \emph{temporal predicates}, an operation which we define in \Cref{temporalize}. 
\begin{definition}\label{temporalize}
Given a temporal graph $G = \{G_1, \cdots, G_T\}$, where each $G_t$ is of the form $(V_t, \mathcal{E}_t, P_1, P_2)$,
% the \emph{temporalisation} of $G$
% is the temporal graph  
\textit{temporal predicates} are 
obtained from $G$ by replacing, for each $t \in \{1, \dots, T\}$ and each $p \in P_2$, each triple $(v_a, p, v_b) \in \mathcal{E}_t$ with $(v_a, p^t,v_b)$, where $p^t$ is a fresh predicate, unique for $p$ and $t$. Similarly, each unary fact $(v_a,q)\in \mathcal{E}_t, q\in P_1$ should be replaced by $(v_a,q^t)$.

% The process of \emph{temporalizing} predicates has been shown in Figure~\ref{fig:tkg}. 
\end{definition}

Note that temporalising introduces $T \times |P|$ new predicates in total. 
By \emph{temporalizing predicates}, we assign a superscript to each predicate and use it to distinguish relations over different timestamps. 
% Note that the superscript is used to represent the relative order between different snapshots rather than the exact time stamp. 
% For example, if we train a \rrgnn in a temporal knowledge graph containing $\{1,\cdots, T\}$ graph snapshots, we can apply the trained \rrgnn to predict over $\{T+1, \cdots, T+T\}$ based on the \emph{window} assumption, namely, the prediction at the current timestamp $t$ is affected by at most $T$ previous time stamps. 
% In theory, a temporal knowledge graph can contain graph snapshot of infinite time stamps, so the \emph{window} assumption is reasonable and practical in terms of the model size and computation efficiency~\cite{jin2019recurrent,pareja2020evolvegcn,park2022evokg,gao2022equivalence}. 

\begin{definition}\label{collapse}
Given a temporal knowledge graph $G = \{G_1, \dots, G_T\}$, the collapse function $H$ maps $G$ to the static graph $H(G)$ obtained by taking the union of graphs over all timestamps in the temporalization of $G$.
%~(obtained by applying \emph{temporalizing predicates} on $G$).
\end{definition}
As we have proved in \Cref{sec:acrgnntg}, for \mg, \rrgnn with \emph{graph transformation} is more powerful than the pure \rrgnn. Here, we transfer these theoretical findings in \mg to the setting of temporal knowledge graphs. To be more specific, after \emph{temporalizing predicates},  we apply a \emph{graph transformation} to each graph snapshot.
%\textcolor{red}{also wrong, explained below}

\begin{definition}
Given a classifier $\mathcal{C}$ over temporal knowledge graphs, we define $\mathcal{C} \circ F^T$ to be a new classifier that takes any temporal knowledge graph as input, applies \emph{graph transformation} to each snapshot, and then applies classifier $\mathcal{C}$.
\end{definition}

% \textcolor{red}{Yeyuan: The above definition is wrong: 1.aux1 and aux2 facts have to be aggregated at least for the last timestamp T. 2. $F^T$ doesn't need temporalization, since the number of binary predicates won't increase in TGNN}

\paragraph{R$^2$-TGNN} \citet{gao2022equivalence} casts node representation in temporal graphs into two frameworks: \emph{time-and-graph} and \emph{time-then-graph}. Due to space constraints, we refer interested readers to \cite{gao2022equivalence}  for more details about the two frameworks. Here, we define a more general GNN-based framework abbreviated as R$^2$-TGNN, where each R$^2$-TGNN is a sequence $\{\mathcal{A}_t\}_{t=1}^{T}$, where each $\mathcal{A}_t$ is an R$^2$-GNN model.Given a temporal knowledge graph $G=\{G_1, \dots, G_T\}$, where $G_t = (V_t, \mathcal{E}_t, P_1, P_2)$ for each $t \in \{1, \dots, T\}$. The updating rule is as follows:
\begin{equation}\label{acrtgnn}
\textbf{x}_v^t=\mathcal{A}_t\biggl(G_t,v,\textbf{y}^{t}\biggl) \text{~~~~~~where~~~~~}\textbf{y}_v^{t}=[ I_{G_t}(v):\textbf{x}_v^{t-1}], \forall v\in V(G_t) 
\end{equation}

where $I_{G_t}(v)$ is the one-hot initial feature vector of node $v$ at timestamp $t$,
and $\mathcal{A}_t(G_t,v,\textbf{y}^t)$ calculates the new feature vector of $v$ by running the R$^2$-GNN model $\mathcal{A}_t$ on $G_t$, but using $\textbf{y}^t$ as the initial feature vectors. As shown in \Cref{tgnnrelation},  R$^2$-TGNN composed with $F^T$ have the same expressiveness as \emph{time-then-graph}\footnote{Since temporalized predicates and timestamps make the definitions of bounded/simple/universal graph class vague, we no longer distinguish them in temporal settings. In theorem statements of this section, $=,\subseteq$ always hold for any temproral graph class, and $\subsetneq,\nsubseteq$ hold for some temporal graph class}, while being more powerful than \emph{time-and-graph}. 

\begin{theorem}\label{tgnnrelation}
time-and-graph $\subsetneq$  R$^2$-TGNN $\circ F^T=$ time-then-graph.
\end{theorem}

We also establish the validity of \Cref{ttransformation}, which asserts that R$^2$-TGNN with \emph{graph transformation} maintains the same expressive power, whether it is applied directly to the temporal graph or to the equivalent collapsed static multi-relational graph

% We also prove \Cref{ttransformation}, namely, that R$^2$-TGNN with \emph{graph transformation} has the same expressive power whether we apply it directly to the temporal graph or to the equivalent collapsed static multi-relational graph.
% ~(defined in \Cref{collapse}).

\begin{theorem}\label{ttransformation}
R$^2$-TGNN $\circ F^T=$  R$^2$-GNN $\circ F \circ H$
\end{theorem}
\begin{figure}
    \centering
   
\begin{tikzpicture}[scale=0.80]
     
     \draw[draw=black] (-1,0.5) rectangle ++(12.5,-3);
		\node[fill=none,draw=none] (a) at (0,0) {R$^2$-TGNN};
		\node[fill=none,draw=none] (b) at (4,-2) {time-and-graph};
		\node[fill=none,draw=none] (c) at (4,0) {R$^2$-GNN$\circ H$};
		\node[fill=none,draw=none] (d) at (10,-0.3) {R$^2$-GNN$\circ F \circ H$ };
	    \node[fill=none,draw=none] (e) at (10,-0.9) { R$^2$-TGNN $\circ F^T$ };
	    \node[fill=none,draw=none] (f) at (10,-1.5) { time-then-graph};
	   	\draw[->] (a) to node[midway,above] {$\subsetneq$} (c);
	   	\draw[->] (c) to node[midway,above] {$\subsetneq$} (e);
	   	\draw[->] (b) to node[midway,below] {$\subsetneq$} (e);
     \draw[->] (a) to node[midway,above] {$\nsubseteq$} (b);
\end{tikzpicture}
\caption{Hierarchic expressiveness.\label{fig:hierarchy}}
\end{figure}
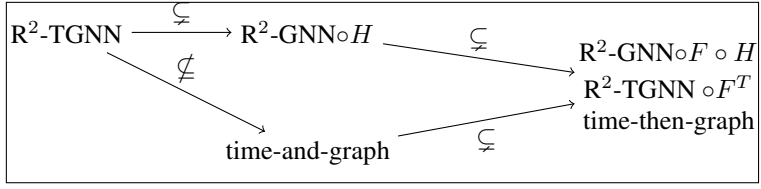
We also prove a strict inclusion that R$^2$-TGNN $\subsetneq$ R$^2$-TGNN$\circ H$. Finally we get the following hierarchy of these frameworks as in Figure~\ref{fig:hierarchy}. the proof of \Cref{hierarchymain} is in the appendix.

\begin{theorem}\label{hierarchymain}
The following hold:
\begin{compactitem}
\item R$^2$-GNN $\subsetneq$ R$^2$-GNN $\circ H$ $\subsetneq$ R$^2$-TGNN $\circ F\circ H$= R$^2$-TGNN $\circ F^T$= time-then-graph.

\item time-and-graph $\subsetneq$  R$^2$-TGNN $\circ F^T$. 

\item R$^2$-TGNN $\nsubseteq$ time-and-graph.
\end{compactitem}
\end{theorem}

\section{Experiment}
% \begin{table}[t!]
% \label{tab:synthetict}
%      \centering
%     \begin{tabular}{c|c|c|c}
%         \hline
%         datasets & AC-TGNN & ACR-TGNN & ACR-TGNN $\circ F^T$ \\
%         \hline
%         $\varphi_1$ & 100/60.7/65.4 & 100/63.5/66.8 & 100/67.2/68.1 \\ 
%         \hline 
%         $\varphi_2$ & 61.0/51.3/52.4 & 93.1/57.7/60.2 & 99.0/57.6/62.2 \\ 
%         \hline 
%         $\varphi_3$ & 93.7/82.3/84.4 & 94.5/83.3/85.9 & 100/88.8/89.2 \\
%         \hline
%         $\varphi_4$ & 83.5/60.0/61.3 & 85.0/62.3/66.2 & 98.1/73.4/77.5 \\
%         \hline
%     \end{tabular}
%     \caption{Accuracy Results of Temporal GNN methods on Synthetic Dataset.}
% \end{table}

% \begin{table}[t!]
% \label{tab:synthetics}
%      \centering
%     \begin{tabular}{c|c|c|c}
%         \hline
%         datasets & AC-GNN $\circ H$ & ACR-GNN $\circ H$ & ACR-GNN $\circ F\circ H$ \\
%         \hline
%         $\varphi_1$ & 100/61.2/69.9 & 100/62.7/66.8 & 100/70.2/70.8 \\ 
%         \hline 
%         $\varphi_2$ & 62.3/51.3/55/5 & 92.4/56.3/58.5 & 98.8/60.6/60.2 \\ 
%         \hline 
%         $\varphi_3$ & 94.7/80.5/83.2 & 95.5/84.2/85.2 & 100/85.6/86.5 \\
%         \hline
%         $\varphi_4$ & 80.2/60.1/60.4 & 81.0/58.3/64.5 & 95.5/70.3/79.7 \\
%         \hline
%     \end{tabular}
%     \caption{Accuracy Results of static GNN methods on Synthetic Dataset.}
% \end{table}
We empirically verify our theoretical findings for \mg by evaluating and comparing the testing performance of \rrgnn with \emph{graph transformation} and less powerful GNNs~(\rgnn and \rrgnn). We did two groups of experiments 
% on static multi-graphs and temporal graphs respectively. For each group, we've considered 
on synthetic datasets and real-world datasets, respectively. Details for datasets generation and statistical information as well as hyper-parameters can be found in the Appendix.

\subsection{Synthetic Datasets}
We first define three simple $\foc$ classifiers
\begin{align*}
    \mathbf{\varphi_1}&\coloneqq  \exists^{\ge2} y(p_1^1(x,y)\wedge Red^1(y))\wedge \exists^{\ge1} y(p_1^2(x,y)\wedge Blue^2(y)) \\
    \mathbf{\varphi_2}&\coloneqq  \exists^{[10,20]}y (\neg p_1^2(x,y)\wedge\varphi_1(y)) ~~~~~~~~~~
    \mathbf{\varphi_3} \coloneqq \exists^{\ge 2}y(p_1^1(x,y)\wedge p_1^2(x,y))
\end{align*}
% $\mathbf{\varphi_1}\coloneqq  \exists^{\ge2} y(p_1^1(x,y)\wedge Red^1(y))\wedge \exists^{\ge1} y(p_1^2(x,y)\wedge Blue^2(y)) ~~||~~  \mathbf{\varphi_2}\coloneqq  \exists^{[10,20]}y (\neg p_1^2(x,y)\wedge\varphi_1(y))~~||~~ \mathbf{\varphi_3} \coloneqq \exists^{\ge 2}y(p_1^1(x,y)\wedge p_1^2(x,y))$.
Besides, we define another complicate $\foc$ classifier denoted as $\varphi_4$ shown as follows:

\begin{multline*}\label{classifier4}
\mathbf{\varphi_4} \coloneqq\bigvee_{3\leq t\leq 10}(\exists^{\ge 2}y(Black^t(y)\wedge Red^{t-1}(y)\wedge Blue^{t-2}(y)\wedge p_1^t(x,y)\wedge p_2^{t-1}(x,y)\wedge p_3^{t-2}(x,y)\wedge\varphi^t(y)) \\
\text{~~~~~~~~~~~~~where~~~~}
\varphi^t(y)\coloneqq \exists^{\ge2} x(p_1^t(x,y)\wedge Red^t(x))\wedge \exists^{\ge1} x(p_2^{t-1}(x,y)\wedge Blue^{t-2}(x))
\end{multline*}

For each of them, we generate an independent dataset containing 7k \mg of size up to 50-1000 nodes for training and 500 \mg of size similar to the train set. We tried different configurations for the aggregation functions and evaluated the node classification performances of three temporal GNN methods~(R-TGNNs, R$^2$-TGNNs and R$^2$-TGNNs $\circ F^T$) on these datasets. 
%
% and three static GNN methods~(R-GNNs $\circ H$, R$^2$-GNNs $\circ H$ and R$^2$-GNNs $\circ F\circ H$) on these datasets. 
%
% In particular, when evaluating static GNN methods, we first transform all temporal graphs into static graphs by the operation defined in \Cref{collapse}. 

% Consider a temporal knowledge graph with time range $T=2$ and a unary predicate set $\{Red,Blue\}$ and a binary predicate set $\{p_1\}$. The first three logical classifiers are defined as follows,
% \begin{multline}
% \vspace{-2cm}
% \label{eq:classifiers1}
% \mathbf{\varphi_1} \coloneqq  \exists^{\ge2} y(p_1^1(x,y)\wedge Red^1(y))\wedge \exists^{\ge1} y(p_1^2(x,y)\wedge Blue^2(y))~~~~~\\
% \mathbf{\varphi_2} \coloneqq  \exists^{[10,20]}y (\neg p_1^2(x,y)\wedge\varphi_1(y)) ~~~~~~
% \mathbf{\varphi_3} \coloneqq \exists^{\ge 2}y(p_1^1(x,y)\wedge p_1^2(x,y))
% \end{multline}

We verify our hypothesis empirically according to models' actual performances of fitting these three classifiers. Theoretically, $\varphi_1$ should be captured by all three models because the classification result of a node is decided by the information of its neighbor nodes, which can be accomplished by the general neighborhood based aggregation mechanism. $\varphi_2$ should not be captured by R-TGNN because the use of $\neg p_1^2(x,y)$ as a guard means that the classification result of a node depends on the global information including those isolated nodes, which needs a global readout. For $\varphi_3$ and $\varphi_4$, they should only be captured by R$^2$-TGNNs $\circ F^T$. An intuitive explanation for this argument is that if we \emph{temporalise predicates} and then collapse the temporal graph into its equivalent static multi-relational graph using $H$, we will encounter the same issue as in the Figure~\ref{graphcompare}. Thus we can't distinguish expected nodes without \emph{graph transformation}. 

% In particular, we provide a A theoretical explanation could be found  proof could We also provide an 

% (whether R$^2$-TGNN 
%  captures $\varphi_3$ can't be explained by theorems in the main body, but it will be explained by an expressiveness hierarchy in appendix\label{hierneed}).
%$\varphi_4$ can be only captured by R$^2$-TGNNs $\circ F^T$.

% Furthormore, we consider a larger predicate set with time range $T=10$, a unary predicate set $\{Red,Blue,Black\}$ and a binary predicate $\{p_1,p_2,p_3\}$. $\varphi_4$ is defined as follows,

% \begin{multline}\label{classifier4}
% \varphi_4:=\bigvee_{3\leq t\leq 10}(\exists^{\ge 2}y(Black^t(y)\wedge Red^{t-1}(y)\wedge Blue^{t-2}(y)\wedge p_1^t(x,y)\wedge p_2^{t-1}(x,y)\wedge p_3^{t-2}(x,y)\wedge\varphi^t(y)) \\
% \text{~~~~~~~~~~~~~where~~~~}
% \varphi^t(y)\coloneqq \exists^{\ge2} x(p_1^t(x,y)\wedge Red^t(x))\wedge \exists^{\ge1} x(p_2^{t-1}(x,y)\wedge Blue^{t-2}(x))
% \end{multline}

% Theoretically, this classifier can be only captured by ACR-TGNN $\circ F^T$ or ACR-GNN $\circ F\circ H$. It's designed for corroborating our conclusion.
 \begin{table}[]
     \centering
     \small 
     \setlength\tabcolsep{4pt}
     \renewcommand{\arraystretch}{1.1}
     \begin{tabular}{l|ccc|ccc|ccc|ccc}
     \hline
       $\foc$ classifier & \multicolumn{3}{c|}{ $\mathbf{\varphi_1}$ } & \multicolumn{3}{c|}{ $\mathbf{\varphi_2}$} &  \multicolumn{3}{c|}{ $\mathbf{\varphi_3}$ } & 
      \multicolumn{3}{c}{ $\mathbf{\varphi_4}$} 
      \\
      %\cline{2-13}
      Aggregation & sum & max & mean& sum & max & mean& sum & max & mean& sum & max & mean \\ 
      \hline 
      \multicolumn{13}{c}{Temporal Graphs Setting} \\
\hline 
      R-TGNN &  100 & 60.7 & 65.4 & 61.0 & 51.3 & 52.4 & 93.7 & 82.3 & 84.4 & 83.5 & 60.0 & 61.3\\ 
      R$^2$-TGNN & 100 & 63.5 & 66.8 & 93.1 & 57.7 & 60.2 & 94.5 & 83.3 & 85.9 & 85.0 & 62.3 & 66.2\\
      R$^2$-TGNN $\circ F^T$ & \textbf{100} & 67.2 & 68.1 & \textbf{99.0} & 57.6 & 62.2 & \textbf{100} & 88.8 & 89.2 & \textbf{98.1} & 73.4 & 77.5\\
      \hline 
       \multicolumn{13}{c}{Aggregated Static Graphs Setting} \\
      \hline 
     \rgnn $\circ H $ & 100 & 61.2 & 69.9 & 62.3 & 51.3 & 55.5 & 94.7 & 80.5 & 83.2 & 80.2 & 60.1 & 60.4 \\ 
      \rrgnn $\circ H $ & 100 & 62.7 & 66.8 & 92.4 & 56.3 & 58.5 & 95.5 & 84.2 & 85.2 & 81.0 & 58.3 & 64.5 \\
      \rrgnn $\circ F \circ H$ & \textbf{100} & 70.2 & 70.8 & \textbf{98.8} & 60.6 & 60.2 & 
     \textbf{100} & 85.6 & 86.5 7 &  \textbf{95.5} & 70.3 & 79.7 \\
     \hline 
     \end{tabular}
     \caption{Test set node classification accuracies~(\%) on synthetic temporal \mg datasets and their aggregated static \mg  datasets. 
      The best results are highlighted for two different settings.}
     \label{tab:moresynthetic}
 \end{table}
Results for temporal GNN methods and static GNN methods on four synthetic datasets can be found in Table~\ref{tab:moresynthetic}. We can see that 
% for both temporal graphs setting and the aggregated static graphs setting, '
%
\rrgnn  with \emph{graph transformation} achieves the best performance. Our theoretical findings show that it is a more expressive model, and the experiments indeed suggest that the model can exploit this theoretical expressiveness advantage to produce better results. Besides, we can also see that R$^2$-TGNN $\circ F\circ H$ and R$^2$-TGNN $\circ F^T$ achieve almost the same performance, which is in line with \Cref{ttransformation}.
\subsection{Real-world Datasets}
% \begin{table}[t!]
% \label{realstatic}
%      \centering
%     \begin{tabular}{c|c|c|c|c|c|c|c}
%         \hline
%         datasets & AC-GNN $\circ H$ & ACR-GNN & ACR-GNN $\circ F$ & RGCN & Feat & WL & RDF2Vec\\
%         \hline
%         AIFB & 91.7/73.8/82.5 & 91.7/71.2/76.9 & 97.2/75.0/89.2 & 95.8 &55.6 &80.6 & 88.9\\
%         \hline
%         MUTAG & 76.5/63.3/73.2 & 85.3/62.1/79.5 & 88.2/65.5/82.1 & 73.2 & 78.0 & 80.9 & 67.2 \\
%         \hline
%     \end{tabular}
%     \caption{Accuracy Results of static GNN methods on Real-world Dataset.}
% \end{table}

% \begin{table}[t!]
% \label{realtemporal}
%      \centering
%     \begin{tabular}{c|c}
%         \hline 
%         Models & Brain-10 \\
%         \hline
%         AC-TGNN & 91.0/82.3/88.8 \\
%         \hline
%         ACR-TGNN & 94.8/82.3/91.0 \\
%         \hline
%         AC-TGNN $\circ F^T$ & 94.0/83.5/92.5 \\
%         \hline
%         EvolveGCN-O & 62.3 \\
%         \hline
%         EvolveGCN-H & 35.6 \\
%         \hline
%         GCRN-M2 & 86.3 \\
%         \hline
%         WD-GCN & 85.1 \\
%         \hline
%         DySAT & 81.6 \\
%         \hline
%         TGN & 84.2 \\
%         \hline
%         GRU-D-GCN & 85.7.0 \\
%         \hline
%         GRU-D-GTC & 85.6 \\
%         \hline
%         GRU-S-GCN & 86.4 \\
%         \hline
%         GRU-S-GTC & 84.9 \\
%         \hline
%     \end{tabular}
%     \caption{Accuracy Results of Temporal GNN methods on Real-world Dataset.}
% \end{table}
 \begin{table}
     \centering
     \setlength\tabcolsep{3.5pt}
     \renewcommand{\arraystretch}{1.1}
     \captionof{table}{Results on temporal graphs.}
     \label{temporalreal}
     \begin{tabular}{lccccc}
     \toprule 
      \multirow{2}{*}{ Models} &  \multirow{2}{*}{ Category}  &  \multirow{2}{*}{ Source } &  \multicolumn{3}{c}{Brain-10}
      \\
      %\cline{2-13}
        &  &   & sum & max & mean\\ 
      \hline 
%       \multicolumn{6}{c}{Temporal Graphs Setting} \\
% \hline 
       GCRN-M2 & time-and-graph  &  \citet{seo2018structured} &  77.0 & 61.2 & 73.1 \\
       DCRNN & time-and-graph  & \citet{DBLP:conf/iclr/LiYS018}  &  84.0 & 70.1 & 66.5 \\
       TGAT & time-then-graph & \citet{DBLP:conf/iclr/XuRKKA20}& 80.0 & 72.3 & 79.0 \\
       TGN & time-then-graph &  \citet{rossi2020temporal} & 91.2 & \textbf{88.5} & 89.2 \\
       GRU-GCN & time-then-graph & \citet{gao2022equivalence}  &  91.6 & 88.2 & 87.1 \\
     \midrule 
     R-TGNN & --  & --  &   85.0 &82.3&82.8\\ 
      R$^2$-TGNN & --  & --  & \textbf{94.8}& 82.3& 91.0\\
      R$^2$-TGNN $\circ F^T$ & --  & --  & 94.0 &83.5 & \textbf{92.5}\\
      \bottomrule
      \end{tabular}
\end{table}
\begin{wraptable}{r}{9cm}
\centering
     \captionof{table}{Results on two static \mg.}
          \label{staticreal}
     \begin{tabular}{l|ccc|ccc}
     \toprule 
       Models &  \multicolumn{3}{c|}{AIFB} &  \multicolumn{3}{c}{MUTAG}
      \\
      %\cline{2-13}
        &  sum & max & mean& sum & max & mean\\ 

      \midrule 
    
      \rgnn  &     \textbf{91.7}&73.8&82.5 &  \textbf{76.5}&63.3&73.2 \\ 
      \rrgnn &     \textbf{91.7}&73.8&82.5 &  \textbf{85.3}&62.1&79.5 \\
      \rrgnn $\circ F$ &     \textbf{97.2}&75.0&89.2 &  \textbf{88.2}&65.5&82.1 \\
      R-GCN&     \textbf{95.8}&77.9&86.3 &  \textbf{73.2}&65.7&72.1 \\
      \toprule 
      \end{tabular}
 \end{wraptable}
For real-world static \mg benchmarks, we used AIFB and MUTAG from \cite{ristoski2016rdf2vec}. Since open source datasets for the node classification on temporal knowledge graphs are rare, we only tried one dataset Brain-10 \cite{gao2022equivalence} for temporal settings.\footnote{The other three temporal dataset mentioned in \cite{gao2022equivalence} are not released.} 
%
%Description and statistical information for real-world datasets can be found in our appendix.

For static \mg, we compare the performances of our methods with RGCN~\cite{schlichtkrull2018modeling}. 
Note that RGCN assigns each node an index and the initial embedding of each node is initialised based on the node index, so the initialisation functional is not permutation-equivariant~\cite{chen2019equivalence} and RGCN cannot be used to perform an isomorphism test. 
% Although Rtheoretical expressiveness is uncomparable with our \emph{equivariant} models.
However, from Table~\ref{staticreal}, we can see that \rrgnn with \emph{graph transformation} still achieves the highest accuracy while being able to be used for the graph isomorphism test. Besides, \rrgnn $\circ F$ also performs better compared with both \rgnn and \rrgnn. This again suggests that the extra expressive power gained by adding a \emph{graph transformation} step to \rrgnn can be exploited by the model to obtain better results.

For temporal graphs, \citet{gao2022equivalence} have classified existing temporal models into two categories, \emph{time-and-graph} and \emph{time-then-graph}, and shown that \emph{time-then-graph} models have better performance. We choose five models mentioned in \citet{gao2022equivalence} as our baseline and include the best accuracy of the dataset Brain-10 reported in \cite{gao2022equivalence}.
%For the temporal knowledge graph dataset~(Brain-10), we compare our models with  %EvolveGCN~\cite{pareja2020evolvegcn} and 
%
%GRU-GCN~\cite{gao2022equivalence}. 
As we expected, R$^2$-TGNNand R$^2$-TGNN $\circ F^T$ achieve better performance than that of the baseline models and R-TGNN accoring to Table~\ref{temporalreal}. However, we observed that although in theory, R$^2$-TGNN $\circ F^T$ has stronger expressive power than R$^2$-TGNN, we did not see an improvement when using R$^2$-TGNN $\circ F^T$~($0.8\%$ accuracy drop). To some extent, it may show that some commonly used benchmarks are inadequate for testing advanced GNN variants. Similar phenomena have also been observed in previous works~\cite{chen2019powerful,Barcelo2020The}.

\section{Conclusion}
% In this paper, we presented R$^2$-GNNs with/without \emph{graph transformation} and analysed their logical expressiveness  in \mg under differnet situations. We also  transfer our theoretical findings  to the temporal graph setting.  Our experiments confirm our results, especially show our \emph{graph transformation} achieves the state-of-the-art performance.

We analyze expressivity of R$^2$-GNNs with and without \emph{graph transformation} in \mg under different situations. Furthermore, we extend our theoretical findings to the temporal graph setting. Our experimental results confirm our theoretical insights, particularly demonstrating the state-of-the-art performance achieved by our \emph{graph transformation} technique.
\section{Acknowledgements}
The authors extend their gratitude to Bernardo Cuenca Grau and David Tena Cucala for their valuable insights, stimulating discussions, and support.

\bibliographystyle{unsrtnat}
\bibliography{reference}

\newpage
\appendix 
\addcontentsline{toc}{section}{\protect\large Appendix} % Add the appendix text to the document TOC
\part{Appendix} % Start the appendix part
\parttoc % Insert the appendix TOC
\newpage
\section{Preliminaries for Proofs}

In this section, we give some preliminaries which will be used to prove the theorems, propositions and lemmas shown in our main body. In what follows, we fix a unary predicate set $P_1$ and a binary predicate set $P_2$. 

\begin{definition}\label{01acrgnn}
For an R$^2$-GNN, we say it is a  0/1-GNN if the recursive formula used to compute vectors $\mathbf{x}_v^{(i)}$ for each node $v$ in a multi-relational graph $G=\{V, \mathcal{E}, P_1, P_2\}$ on each layer $i$ is in the following form 

% \begin{equation}
% \feature_v^{(i)}=f\biggl(C^{(i)}\Bigl(\feature_v^{(i-1)}+\big(A_j^{(i)}\bigl(\sum_{u\in\mathcal{N}_{G,j}(v)}\feature_u^{(i-1)}\bigl)\big)_{j=1}^{|P_2|}+R^{(i)}\big(\sum_{u\in V(G)}\feature_u^{(i-1)}\big)\Bigl + b^{(i)}\Bigl)\biggl)
% \end{equation}

\begin{equation}
\label{01acrgnn}
    \feature_v^{(i)} = f \left( C^{(i)} \left ( \mathbf{x}_v^{(i-1)} +  \sum_{r\in P_2} \sum_{u\in V} A_r^{(i)} \mathbf{x}_{u}^{(i-1)} + R^{(i)} \left( \sum_{u\in V}\feature_u^{(i-1)} \right)  + b^{(i)} \right) \right)
\end{equation}

% \begin{equation}\label{01acrgnn}
%     \textbf{x}_v^{(i)}=f \left ( C^{(i)}\left( \textbf{x}_v^{(i-1)},\left(A_j^{(i)}(\mulsetl \textbf{x}_u^{(i-1)}|u\in \mathcal{N}_{G,j}(v)\mulsetr)\right)_{j=1}^{|\relations|}, R^{(i)}(\mulsetl \textbf{x}_u^{(i-1)}|u\in V(G)\mulsetr)\right)  \right)
% \end{equation}

%feature vectors $\feature$ are all seen as column vectors, and
where $C^{(i)},A^{(i)}_j,R^{(i)}$ are all integer matrices of size $d_{i}\times d_{i-1}$, $b^{(i)}$ is bias column vector with size $d_i\times 1$, where $d_{i-1}$ and $d_i$ are input/output dimensions, and $f$ is defined as $max(0,min(x,1))$.

Furthermore, we restrict the final output dimension be $d_L=1$. Since all matrices have integer elements, initial vectors are integer vectors by initialisation function $I(\cdot)$ (\Cref{nodeenc}), and $max(0,min(x,1))$ will map all integers to $0/1$, it's easy to see that the output of this kind of model is always $0/1$, which can be directly used as the classification result. We call such model 0/1-GNN. A model instance can be represented by $\{C^{(i)},(A^{(i)}_j)_{j=1}^K,R^{(i)},b^{(i)}\}_{i=1}^L$, where $K=|P_2|$
\end{definition}

\begin{lemma}\label{closure}
Regard 0/1-GNN as node classifier, then the set of node classifiers represented by 0/1-GNN is closed under $\wedge,\vee,\neg$.
\end{lemma}
\begin{proof}
Given two 0/1-GNN $\mathcal{A}_1$,$\mathcal{A}_2$, it suffices to show that we can construct $\neg \mathcal{A}_1$ and $\mathcal{A}_1\wedge \mathcal{A}_2$ in 0/1-GNN framework. That's because construction of $ \mathcal{A}_1\vee\mathcal{A}_2$ can be reduced to constructions of $\wedge,\neg$ by De Morgan’s law, e.g., $a\vee b=\neg(\neg a\wedge \neg b)$.

1. Construct $\neg \mathcal{A}_1$. Append a new layer to $\mathcal{A}_1$ with dimension $d_{L+1}=1$. For matrices and bias $C^{(L+1)},(A^{(L+1)}_j)_{j=1}^K,R^{(L+1)},b^{(L+1)}$ in layer $L+1$, set $C^{L+1}_{1,1}=-1$ and $b^{L+1}_1=1$ and other parameters $0$. Then it follows $\feature_v^{(L+1)}=max(0,min(-\feature_v^{(L)}+1,1))$. Since $\feature_v^{(L)}$ is the 0/1 classification result outputted by $\mathcal{A}_1$. It's easy to see that the above equation is exactly $\feature_v^{(L+1)}=\neg \feature_v^{(L)}$

2. Construct $\mathcal{A}_1\wedge \mathcal{A}_2$. Without loss of generality, we can assume two models have same layer number $L$ and same feature dimension $d_l$ in each layer $l\in\{1......L\}$. Then, we can construct a new 0/1-GNN $\mathcal{A}$. $\mathcal{A}$ has $L+1$ layers. For each of the first $L$ layers, say $l$-th layer, it has feature dimension $2d_l$. Let $\{C^{(l)}_1,(A^{(l)}_{j,1})_{j=1}^K,R^{(l)}_1,b^{(l)}_1\},\{C^{(l)}_2,(A^{(l)}_{j,2})_{j=1}^K,R^{(l)}_2,b^{(l)}_2\}$ be parameters in layer $l$ of $\mathcal{A}_1,\mathcal{A}_2$ respectively. Parameters for layer $l$ of $\mathcal{A}$ are defined below
\begin{equation}\label{parconcat}
\textbf{C}^{(l)}:=\begin{bmatrix} \textbf{C}^{(l)}_1& \\ & \textbf{C}^{(l)}_2\end{bmatrix}
\textbf{A}_j^{(l)}:=\begin{bmatrix} \textbf{A}_{j,1}^{(l)}& \\ & \textbf{A}_{j,2}^{(l)}\end{bmatrix}
\textbf{R}^{(l)}:=\begin{bmatrix} \textbf{R}^{(l)}_1& \\ & \textbf{R}^{(l)}_2\end{bmatrix}
\textbf{b}^{(l)}:=\begin{bmatrix} \textbf{b}_1^{(l)}\\ \textbf{b}_2^{(l)}\end{bmatrix}
\end{equation}
Initialization function of $\mathcal{A}$ is concatenation of initial feature of $\mathcal{A}_1,\mathcal{A}_2$. Then it's easy to see that the feature $\textbf{x}^{L}_v$ after running first $L$ layers of $\mathcal{A}$ is a two dimension vector, and the two dimensions contains two values representing the classification results outputted by $\mathcal{A}_1,\mathcal{A}_2$ respectively.

For the last layer $L+1$, it has only one output dimension. We just set  $\textbf{C}^{L+1}_{1,1}=\textbf{C}^{L+1}_{1,2}=1,\textbf{b}^{L+1}_1=-1$ and all other parameters $0$. Then it's equivalent to $\feature_v^{(L+1)}=max(0,min(\feature_{v,1}^{(L)}+\feature_{v,2}^{(L)}-1,1))$ where $\feature_{v,1}^{(L)},\feature_{v,2}^{(L)}$ are output of $\mathcal{A}_1,\mathcal{A}_2$ respectively. It's easy to see that the above equation is equivalent to $\feature_v^{(L+1)}=\feature_{v,1}^{(L)}\wedge\feature_{v,2}^{(L)}$ so the $\mathcal{A}$ constructed in this way is exactly $\mathcal{A}_1\wedge\mathcal{A}_2$
\end{proof}

\begin{definition}\label{foc2}
% Given a set of unary predicates $P_1= \{A_1, \dots, A_m\}$ and a set of binary predicates $P_2 = \{r_1, \dots, r_K\}$ , where $m$ and $K$ denotes the number of unary predicates and binary predicates, respectively.
A $\foc$  formula is defined inductively according to the following grammar:
\begin{equation}\label{focdef1}
A(x), r(x,y), \varphi_1\wedge\varphi_2,  \varphi_1\vee\varphi_2,  \neg\varphi_1,\exists^{\ge n}y(\varphi_1(x,y)) \text{ where } A\in P_1 \text{ and } r\in P_2
\end{equation}
\end{definition}

\begin{definition}\label{rsfoc}
% Given a set of unary predicates $P_1= \{A_1, \dots, A_m\}$ and a set of binary predicates $P_2\{r_1, \dots, r_K\}$ , where $m$ and $K$ denotes the number of unary predicates and binary predicates, respectively.
For any subset $S \subseteq P_2$, let $\varphi_S(x,y)$ denote the $\foc$ formula $(\bigwedge_{r\in S}r(x,y))\wedge(\bigwedge_{r\in P_2\setminus S}\neg r(x,y))$. Note that $\varphi_S(x,y)$ means there is a relation $r$ between $x$ and $y$ if and only if $r\in S$, so $\varphi_S(x,y)$ can be seen as a formula to restrict specific relation distribution between two nodes. $\rsfoc$ is inductively defined according to the following grammar:
\begin{equation}\label{rsfocdef1}
A(x), \varphi_1 \wedge\varphi_2,\ \ \varphi_1\vee\varphi_2,\ \ \neg\varphi_1,
\exists^{\ge n}y\bigg(\varphi_S(x,y)\wedge\varphi_1(y)\bigg) \text{ where } A\in P_1 \text{ and } S\subseteq P_2
\end{equation}
\end{definition}

Next, we prove that $\foc$ and $\rsfoc$ have the same expressiveness, namely, each $\foc$ node classifier can be rewritten in the form $\rsfoc$. 

\begin{lemma}\label{rsfoceqfoc}
$\foc=\rsfoc$.
\end{lemma}
\begin{proof}
Comparing the definitions of $\rsfoc$ and $\foc$, it is obvious that  $\rsfoc\subseteq\foc$ trivially holds, so we only need to prove the other direction, namely,  $\foc\subseteq \rsfoc$. In particular, a Boolean logical classifier only contains one free variable, we only need to prove that for any one-free-variable $\foc$ formula $\varphi(x)$, we can construct an equivalent $\rsfoc$ formula $\psi(x)$.

We prove \Cref{rsfoceqfoc} by induction over $k$, where $k$ is the quantifier depth of $\varphi(x)$.

In the base case where $k=0$, $\varphi(x)$ is just the result of applying conjunction, disjunction or negation to a bunch of unary predicates $A(x)$, where $A \in P_1$. Given that the grammar of generating $\varphi(x)$ is the same in $\rsfoc$ and $\foc$ when $k=0$, so the lemma holds for $k=0$.

For the indutive step, we assume that \Cref{rsfoceqfoc} holds for all $\rsfoc$ formula with quantifier depth no more than $m$, we next need to consider the case when $k=m+1$.

We can decompose $\varphi(x)$ to be boolean combination of a bunch of $\foc$ formulas $\varphi_1(x),\dots,\varphi_N(x)$, each of which is in the form $\varphi_i(x):=A(x) \text{ where } A \in P_1$ or $\varphi_i(x):=\exists^{\ge n}y(\varphi'(x,y))$. See the following example for reference.
\begin{example}
Assume $\varphi(x):=\bigl(A_1(x)\wedge\exists y(r_1(x,y))\bigl)\vee\bigl(\exists y\bigl(A_2(y)\wedge r_2(x,y)\bigl)\wedge \exists y(r_3(x,y))\bigl)$. It can be decomposed into boolean combination of four subformulas  shown as follows:
\begin{compactitem}
    \item $\varphi_1(x)=A_1(x)$
    \item $\varphi_2(x)=\exists y(r_1(x,y))$
    \item $\varphi_3(x)=\exists y\bigl(A_2(y)\wedge r_2(x,y)\bigl)$
    \item $\varphi_4(x)=\exists y(r_3(x,y))$
\end{compactitem}
\end{example}

 We can see that grammars of $\foc$ and $\rsfoc$ have a common part: $A(x), \varphi_1\wedge\varphi_2,  \varphi_1\vee\varphi_2,  \neg\varphi_1$, so we can only focus on those subformulas $\varphi_i(x)$ in the form of $\exists^{\ge n}y\varphi'(x,y)$. In other words, if we can rewrite these $\foc$ subformulas into another form satisfying the grammar of $\rsfoc$, we can naturally construct the desired $\rsfoc$ formula $\psi(x)$ equivalent to $\foc$ formula $\varphi(x)$. 
 
 Without loss of generality, in what follows, we  consider the construction for $\varphi(x)=\exists^{\ge n}y(\varphi'(x,y))$. Note that $\varphi(x)$ has quantifier depth no more than $m+1$, and $\varphi'(x,y)$ has quantifier depth no more than $m$.

We can decompose $\varphi'(x,y)$ into three sets of subformulas
$\{\varphi^x_i(x)\}_{i=1}^{N_x},\{\varphi^y_i(y)\}_{i=1}^{N_y},\{r_i(x,y)\}_{i=1}^{|P_2|}$, where $N_x$ and $N_y$ are two natural numbers,  $\varphi^x_i,\varphi^y_i$ are its maximal subformulas whose free variable is assigned to $x$ and $y$, respectively. $\varphi'(x)$ is the combination of these sets of subformulas using $\wedge,\vee,\neg$.

\begin{example}
Assume that  we have a $\foc$ formula in the form of $\varphi'(x,y)=\Bigl(r_1(x,y)\wedge\exists x(r_2(x,y))\Bigl) \vee \Bigl(\exists y\bigl(\exists x (r_3(x,y)) \vee \exists y (r_1(x,y)) \bigl) \wedge \exists y\bigl(A_2(y) \wedge r_2(x,y)\bigl)\Bigl)$

It can be decomposed into the following subformulas:
\begin{compactitem}
    \item $\varphi_1^x(x):=\exists y\bigl(\exists x (r_3(x,y))\vee\exists y (r_1(x,y))\bigl)$;
    \item $\varphi_2^x(x):=\exists y\bigl(A_2(y)\wedge r_2(x,y)\bigl)$;
    \item $\varphi_1^y(y):=\exists x(r_2(x,y))$;
    \item $r_1(x,y)$
\end{compactitem}
\end{example}

Assume that $N:=\{1, \dots, N_x\}$, we construct a $\rsfoc$ formula 
$\varphi^x_{T}(x):=(\bigwedge_{i\in T}\varphi_i^x(x))\wedge(\bigwedge_{i\in N\setminus T}\neg \varphi^x_i(x))$, where $T\subseteq N$. It is called the \emph{x-specification} formula, which means $\varphi^x_T(x)$ is \emph{true} iff the following condition holds: for all 
$i\in T$, $\varphi^x_i(x)$ is \emph{true} and for all $i \in N\setminus T$, $\varphi^x_i(x)$ is \emph{false}.

By decomposing $\varphi'(x,y)$ into three subformula sets, we know Boolean value of $\varphi'(x,y)$ can be decided by Boolean values of these formulas $\{\varphi^x_i(x)\}_{i=1}^{N_x},\{\varphi^y_i(y)\}_{i=1}^{N_y},\{r_i(x,y)\}_{i=1}^{|P_2|}$. Now for any two specific subsets $S\subseteq P_2, T\subset N$, we assume $\varphi_S(x,y)$ and $\varphi^x_T(x)$ are all \emph{true} (Recall the definition of $\varphi_S(x,y)$ in \Cref{rsfoc}). Then Boolean values for formulas in $\{\varphi^x_i(x)\}_{i=1}^{N_x},\{r_i(x,y)\}_{i=1}^{|P_2|}$ are determined and Boolean value of $\varphi'(x,y)$ depends only on Boolean values of $\{\varphi^y_i(y)\}_{i=1}^{N_y}$. Therefore, we can write a  new $\foc$ formula $\varphi^y_{S,T}(y)$ which is a boolean combination of $\{\varphi^y_i(y)\}_{i=1}^{N_y}$. This formula should satisfy the following condition: For any graph $G$ and two nodes $a,b$ on it, the following holds,
\begin{equation}\label{rsfoc1}
\varphi_S(a,b)\wedge\varphi^x_T(a)\Rightarrow\Bigl(\varphi'(a,b)\Leftrightarrow \varphi^y_{S,T}(b)\Bigl)
\end{equation}
By our inductive assumption, $\varphi'(x,y)$ has a quantifier depth  which is no more than $m$, so $\{\varphi^x_i(y)\}_{i=1}^{N_x},\{\varphi^y_i(y)\}_{i=1}^{N_y}$ also have quantifier depths no more than $m$. Therefore, each of them has $\rsfoc$ correspondence. Furthermore, since $\wedge,\vee,\neg$ are allowed operation in $\rsfoc$, $\varphi^x_T(x)$ and $\varphi^y_{S,T}(y)$ can also be rewritten as $\rsfoc$ formulas.

Given that $\varphi_S(x,y)$ and $\varphi^x_T(y)$ specify the  boolean values for all $\{\varphi^x_i(y)\}_{i=1}^{N_x},\{\varphi^r_i(x,y)\}_{i=1}^{|P_2|}$ formulas, so we can enumerate all possibilities over $S\subseteq P_2$ and $T\subseteq N$. Obviously for any graph $G$ and a node pair $(a,b)$, there exists an unique $(S,T)$ pair such that $\varphi_S(a,b)\wedge\varphi^x_T(a)$ holds.

Hence, combining \Cref{rsfoc1}, $\varphi'(x,y)$ is true only when there exists a $(S,T)$ pair such that $\varphi_S(x,y)\wedge\varphi^x_T(x)\wedge\varphi^y_{S,T}(y)$ is \emph{true}. Formally, we can rewrite $\varphi'(x,y)$ as following form:
\begin{equation}\label{rsfoc2}
\varphi'(x,y)\equiv \bigvee_{S\subseteq P_2,T\subseteq N}\Bigl(\varphi_S(x,y)\wedge\varphi^x_T(x)\wedge\varphi^y_{S,T}(y)\Bigl)
\end{equation}
In order to simplify the formula above, let $\phi_T(x)$ denote the following formula:
\begin{equation}
\phi_T(x,y) \coloneqq \bigvee_{S\subseteq P_2}\Bigl(\varphi_S(x,y)\wedge\varphi^y_{S,T}(y)\Bigl)
\end{equation}
Then we can simplify \Cref{rsfoc2} to the following form:
\begin{equation}
\varphi'(x,y)\equiv\bigvee_{T\subseteq N}\Bigl(\varphi^x_T(x)\wedge\phi_{T}(x,y)\Bigl)
\end{equation}
Recall that $\varphi(x)=\exists^{\ge n}y(\varphi'(x,y))$, so it can be rewritten as:
\begin{equation}\label{rsfocsim}
\varphi(x)\equiv\exists^{\ge n}y\biggl(\bigvee_{T\subseteq N}\bigl(\varphi^x_T(x)\wedge\phi_{T}(x,y)\bigl)\biggl)
\end{equation}
Since for any graph $G$ and its node $a$, there exists exactly one  $T$ such that $\varphi^x_T(a)$ is \emph{true}. Therefore, \Cref{rsfocsim} can be rewritten as  the following formula:

\begin{equation}\label{varphi1}
\varphi(x)\equiv \bigvee_{T\subseteq N}\biggl(\varphi_T^x(x)\wedge \exists^{\ge n}y(\phi_T(x,y))\biggl)
\end{equation}

Let $\widehat{\varphi}_T(x) \coloneqq \exists^{\ge n}y(\phi_T(x,y))$. Since $\wedge,\vee$ are both allowed in $\rsfoc$. If we want to rewrite $\varphi(x)$ in the $\rsfoc$ form, it suffices to rewrite $\widehat{\varphi}_T(x)$ as a $\rsfoc$ formula, which is shown as follows,
% Let's recall the definition of $\widehat{\varphi}_T(x)$:
\begin{equation}\label{rsfochat}
\widehat{\varphi}_T(x) \coloneqq \exists^{\ge n}y(\phi_T(x,y))=\exists^{\ge n}y\biggl(\bigvee_{S\subseteq P_2}\Bigl(\varphi_S(x,y)\wedge\varphi^y_{S,T}(y)\Bigl)\biggl)
\end{equation}
Similar to the previous argument, since for any graph $G$ and of of its node pairs $(a,b)$, the \emph{relation-specification} formula $\varphi_S(x,y)$ restricts exactly which types of relations exists between $(a,b)$, there is exactly one subset $S\subseteq P_2$ such that $\varphi_S(a,b)$ holds.

Therefore, for all $S\subseteq P_2$, we can define $n_S$ as the number of nodes $y$ such that $\varphi_S(x,y)\wedge\varphi^y_{S,T}(y)$ holds. Since for two different subsets $S_1,S_2\subseteq P_2$ and a fixed $y$, $\varphi_{S_1}(x,y)$ and $\varphi_{S_2}(x,y)$ can't hold simultaneously, the number of nodes $y$ that satisfies $\varphi_S(x,y)\wedge\varphi^y_{S,T}(y)$ is exactly the sum $\sum_{S\subseteq P_2}n_S$. Therefore, in order to express \Cref{rsfochat}, which means there exists at least $n$ nodes $y$ such that $\bigvee_{S\subseteq P_2}\bigl(\varphi_S(x,y)\wedge\varphi^y_{S,T}(y)\bigl)$ holds, it suffices to enumerate all possible values for $\{n_S|S\subseteq P_2\}$ that satisfies $(\sum_{S\subseteq P_2}n_S ) = n, n_S\in \mathbb{N}$.
% , and say things like 
% "for each $S$, there is at least $n_S$ nodes $y$ such that $\varphi_S(x,y)\wedge\varphi^y_{S,T}(y)$ hold".
Formally, we can rewrite $\widehat{\varphi}_T(x)$ as follows:
\begin{equation}
\widehat{\varphi}_T(x)\equiv \bigvee_{(\sum_{S\subseteq P_2}n_S)=n}\Bigl(\bigwedge_{S\subseteq P_2}\exists^{\ge n_{S}}y(\varphi_S(x,y)\wedge \varphi_{S,T}^y(y))\Bigl)
\end{equation}

Note that $\exists^{\ge n_{S}}y(\varphi_S(x,y)\wedge \varphi_{S,T}^y(y))$ satisfies the grammar of $\rsfoc$, so $\widehat{\varphi}_T(x)$ can be rewritten as $\rsfoc$. Then, since 
$\varphi^x_T(x)$ can also be rewritten as $\rsfoc$ by induction, combining \Cref{varphi1} and \Cref{rsfochat}, $\varphi(x)$ is in $\rsfoc$.
%
% We've shown that construction for $\varphi(x)$ is enough for construction for $\varphi(x)$, 
We finish the proof.
\end{proof}
\section{Proof of \Cref{inclusion}}
\begin{repproposition}{inclusion}
$\foc \not\subseteq $ \rrgnn and \rrgnn $\not\subseteq \foc$  on  some \emph{universal} graph class $\ggraph$.
\end{repproposition}
\begin{proof}
First, we prove $\foc\nsubseteq$ \rrgnn.

Consider the two graphs $G_1,G_2$ in Figure~\ref{graphcompare}. $(G_1,a),(G_2,a)$ can be distinguished by the $\foc$ formula $\varphi(x):=\exists^{\ge 1}y(p_1(x,y)\wedge p_2(x,y))$. However, we will prove that any R$^2$-GNN can't distinguish any node in $G_1$ from any node in $G_2$.

Let's prove it by induction over the layer number $L$ of R$^2$-GNN. That's to say, we want to show that for any $L\ge 0$, R$^2$-GNN with no more than $L$ layers can't distinguish any node of $G_1$ from that of $G_2$.

For the base case  where $L=0$, since each node feature vector is initialized by the unary predicate information, so the result trivially holds.

Assume any R$^2$-GNN with no more than $L=m$ layers can't distinguish nodes of $G_1$ from nodes of $G_2$. Then we want to prove the result for $L=m+1$.

For any R$^2$-GNN model $\mathcal{A}$ with $m+1$ layers, let $\mathcal{A}'$ denote its first $m$ layers, we know outputs of $\mathcal{A}'$ on any node from $G_1$ or $G_2$ are the same,  suppose the common output feature is $\feature^{(m)}$.

Recall the updating rule of R$^2$-GNN in \Cref{acrgnn}.We know the output of $\mathcal{A}$ on any node $v$ in $G_1$ or $G_2$ is defined as follows,
\begin{equation}\label{update}
\textbf{x}_v^{(m+1)}=C^{(m+1)}\biggl( \textbf{x}_v^{(m)},\Bigl(A_1^{(m+1)}(\mulsetl \feature_{u_1(v)}^{(m)}\mulsetr)\Bigl),A_2^{(m+1)}(\mulsetl \feature_{u_2(v)}^{(m)}\mulsetr)\Bigl), R^{(m+1)}(\mulsetl\feature_a^{(m)},\feature_b^{(m)},\feature_c^{(m)},\feature_d^{(m)} \mulsetr)\biggl)
\end{equation}
Here $C^{(m+1)},A_1^{(m+1)},A_2^{(m+1)},R^{(m+1)}$ are parameters in the layer $m+1$ of $\mathcal{A}$,  $u_1(v),u_2(v)$ is the only $r_1,r_2$-type neighbor of $v$, and $a,b,c,d$ are nodes from the corresponding graph $G_1$ or $G_2$.  From Figure~\ref{graphcompare} we can see they are well defined.

By induction, since any node pairs from $G_1$ and $G_2$ can't be distinguished by $\mathcal{A}'$, we have  $\feature_v^{(m)},\feature_{u_1(v)}^{(m)},\feature_{u_2(v)}^{(m)},\feature_a^{(m)},\feature_b^{(m)},\feature_c^{(m)},\feature_d^{(m)}$ are all the same feature $\feature^{(m)}$. Therefore, \Cref{update} have the same expression for all nodes $v$ from $G_1$ and $G_2$, which implies any $\mathcal{A}$ with $m+1$ layers can't distinguish nodes from $G_1$ and $G_2$.

% \begin{proof}
% Consider the $\foc$ formula $\varphi(x)\coloneqq \exists^{\geq 1}y(p_1(x,y) \cap p_2(x,y))$ shown in Figure~\ref{graphcompare}. We will show by contradiction that there is no R$^2$-GNN can captures $\varphi$.  
% \end{proof}
%obvious from Figure~\ref{graphcompare}. 

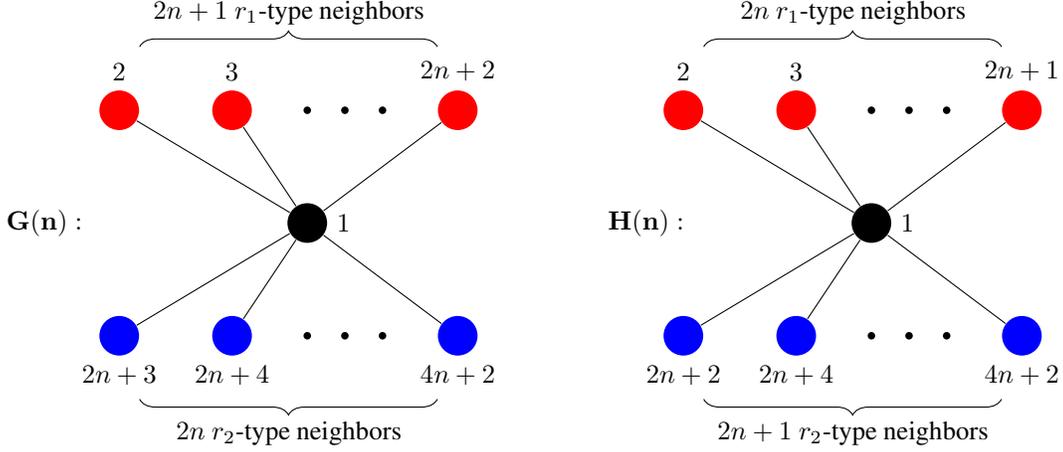
\begin{figure}[t!]
  \centering

  \begin{tikzpicture}

        \node[circle, fill=red,  inner sep=0pt,minimum size=15pt, label=above:{$2$}] (t1) at (0,0) {};

        \node[circle, fill=red, inner sep=0pt,minimum size=15pt, label=above:{$3$}] (t2) at (1.5, 0) {};

        \foreach \x in {2.5, 3, 3.5}
            \fill (\x,0) circle (0.05cm);

        \node[circle, fill=red, inner sep=0pt,minimum size=15pt, label=above:{$2n+2$}] (t3) at (4.5, 0) {};

        \node[circle, fill=black,  inner sep=0pt,minimum size=15pt, label=right:{$1$}] (center1) at (2.5, -1.5) {};

        \node[circle, fill=blue, inner sep=0pt,minimum size=15pt, label=below:{$2n+3$}] (t4) at (0, -3) {};

        \node[circle, fill=blue, inner sep=0pt,minimum size=15pt, label=below:{$2n+4$}] (t5) at (1.5, -3) {};

        \foreach \x in {2.5, 3, 3.5}
            \fill (\x,-3) circle (0.05cm);

        \node[circle, fill=blue, inner sep=0pt,minimum size=15pt, label=below:{$4n+2$}] (t6) at (4.5, -3) {};

        \draw [decorate,decoration={brace,amplitude=6pt,raise=24pt,mirror}]
        (t4) -- (t6) node [black,midway,yshift=-1.3cm] 
       {$2n$ $r_2$-type neighbors};

        \draw [decorate,decoration={brace,amplitude=6pt,raise=24pt}]
        (t1) -- (t3) node [black,midway,yshift=1.3cm] 
       {$2n+1$ $r_1$-type neighbors};

       \draw (t1) -- (center1);
       \draw (t2) -- (center1);
       \draw (t3) -- (center1);

       \draw (t4) -- (center1);
       \draw (t5) -- (center1);
       \draw (t6) -- (center1);

       \node[text width=4cm] at (0.5,-1.5) {$\mathbf{G(n)}:$};

        \node[circle, fill=red,  inner sep=0pt,minimum size=15pt, label=above:{$2$}] (t1) at (7.5,0) {};

        \node[circle, fill=red, inner sep=0pt,minimum size=15pt, label=above:{$3$}] (t2) at (9, 0) {};

        \foreach \x in {10, 10.5, 11}
            \fill (\x,0) circle (0.05cm);

        \node[circle, fill=red, inner sep=0pt,minimum size=15pt, label=above:{$2n+1$}] (t3) at (12, 0) {};

        \node[circle, fill=black,  inner sep=0pt,minimum size=15pt, label=right:{$1$}] (center1) at (10, -1.5) {};

        \node[circle, fill=blue, inner sep=0pt,minimum size=15pt, label=below:{$2n+2$}] (t4) at (7.5, -3) {};

        \node[circle, fill=blue, inner sep=0pt,minimum size=15pt, label=below:{$2n+4$}] (t5) at (9, -3) {};

        \foreach \x in {10, 10.5, 11}
            \fill (\x,-3) circle (0.05cm);

        \node[circle, fill=blue, inner sep=0pt,minimum size=15pt, label=below:{$4n+2$}] (t6) at (12, -3) {};

        \draw [decorate,decoration={brace,amplitude=6pt,raise=24pt,mirror}]
        (t4) -- (t6) node [black,midway,yshift=-1.3cm] 
       {$2n+1$ $r_2$-type neighbors};

        \draw [decorate,decoration={brace,amplitude=6pt,raise=24pt}]
        (t1) -- (t3) node [black,midway,yshift=1.3cm] 
       {$2n$ $r_1$-type neighbors};

       \draw (t1) -- (center1);
       \draw (t2) -- (center1);
       \draw (t3) -- (center1);

       \draw (t4) -- (center1);
       \draw (t5) -- (center1);
       \draw (t6) -- (center1);

       \node[text width=4cm] at (8.5,-1.5) {$\mathbf{H(n)}:$};

    \end{tikzpicture} 
  \caption{$G(n)$ and $H(n)$.} \label{fig:singlee}
\end{figure}

Next, we then prove R$^2$-GNNs $\nsubseteq\foc$.

Assume we want to construct a classifier $c$ which classifies a node into true iff \emph{the node has a larger number of $r_1$-type neighbors than that of $r_2$-type neighbors.} 

First, we prove that we can construct an 0/1-GNN $\mathcal{A}$ to capture $c$. It only has one layer with parameters $C^{(1)},A^{(1)}_1,A^{(1)}_2,R^{(1)}$, and feature dimension $d_0=d_1=1$. We assume that each node has the same initial feature vector, i.e., $\mathbf{1}$.
%Its initialization function $g$ always return vector $\textbf{1}$, 
%
We set $A_{1,(1,1)}^{(1)}=1,A_{2,(1,1)}^{(1)}=-1$, where $A_{1,(1,1)}^{(1)}$ denotes the only element in $A_{1}^{(1)}$ placed in the first row and first column (similar for $A_{2,(1,1)}^{(1)}$) and all other parameters $0$. It's easy to see that $\mathcal{A}$ is equivalent to our desired classifier $c$ on any graph since we have $\feature_v^{(1)}=max(0,min(1,\sum_{u\in\mathcal{N}_{G,1}(v)}1-\sum_{u\in\mathcal{N}_{G,2}(v)}1))$.
 
 Next, we show $\foc$ can't capture $c$ on $\sgraph$. In order to show that, for any natural number $n$, we can construct two single-edge graphs $G(n),H(n)$ as follows:
\begin{align*}
 V(G(n))&=V(H(n))=\{1,2......4n+2\} \\
 E(G(n))&=\{r_1(1,i)|\forall i\in[2,2n+2]\}\cup\{r_2(1,i)|i\in[2n+3,4n+2]\} \\
 E(H(n))& =\{r_1(1,i)|\forall i\in[2,2n+1]\}\cup\{r_2(1,i)|i\in[2n+2,4n+2]\}
\end{align*}
 We prove the result by contradiction. Assume there is a $\foc$ classifier $\varphi$ that captures the classifier $c$, then it has to classify $(G(n),1)$ as \emph{true} and $(H(n),1)$ as \emph{false} for all natural number $n$. However, in the following we will show that it's impossible, which proves the non-existence of such $\varphi$.

 Suppose threshold numbers used on counting quantifiers of $\varphi$ don't exceed $m$, then we only need to prove that $\varphi$ can't distinguish $(G(m),1),(H(m),1)$, which contradicts our assumption. 

 For simplicity, we use $G,H$ to denote $G(m),H(m)$. In order to prove the above argument. First, we define a \emph{node-classification} function $CLS(\cdot)$ as follows. It has $G$ or $H$ as subscript and a node of $G$ or $H$ as input. 
 \begin{enumerate}
     \item $CLS_G(1)=CLS_H(1)=1$. It means the function returns $1$ when the input is the \emph{center} of $G$ or $H$.
     \item $CLS_G(v_1)=CLS_H(v_2)=2, 
 \forall v_1\in[2,2m+2],\forall v_2\in[2,2m+1]$, which means the function returns $2$ when the input is a $r_1$-neighbor of \emph{center}.
      \item $CLS_G(v_1)=CLS_H(v_2)=3, 
 \forall v_1\in[2m+3,4m+2],\forall v_2\in[2m+2,4m+2]$, which means the function returns $3$ when the input is a $r_2$-neighbor of \emph{center}.
 \end{enumerate}

 \textbf{Claim 1}: Given any $u_1,v_1\in V(G), u_2,v_2\in V(H)$, if $(CLS_G(u_1),CLS_G(v_1))=(CLS_H(u_2),CLS_H(v_2))$, then any $\foc$ formula with threshold numbers no larger than $m$ can't distinguish $(u_1,v_1)$ and $(u_2,v_2)$. 

This claim is enough for our result. We will prove that for any constant $d$ and any $\foc$ formula $\phi$ with threshold numbers no larger than $m$ and quantifier depth $d$, $\phi$ can't distinguish $(u_1,v_1)$ and $(u_2,v_2)$ given that $(CLS_G(u_1),CLS_G(v_1))=(CLS_H(u_2),CLS_H(v_2))$

The result trivially holds for the base case where $d=0$. Now let's assume the result holds for $d\leq k$, we can now prove the inductive case when $d=k+1$.

 Since $\wedge,\vee, \neg, r(x,y)$ trivially follows, we can only consider the case when $\phi(x,y)$ is in the form $\exists^{\ge N}y\phi'(x,y),N\leq m$ or $\exists^{\ge N}x\phi'(x,y),N\leq m$, where $\phi'(x,y)$ is a $\foc$ formula with threshold numbers no more than $m$ and quantifier depth no more than $k$. Since these two forms are symmetrical, without loss of generality, we only consider the case $\exists^{\ge N}y\phi'(x,y),N\leq m$.

 Let $N_1$ denote the number of nodes $v'_1\in V(G)$ such that $(G,u_1,v'_1)\models \phi'$ and $N_2$ denote the number of nodes $v'_2\in V(H)$ such that $(H,u_2.v'_2)\models \phi'$. Let's compare values of $N_1$ and $N_2$. First, By induction, since we have
 $CLS_G(u_1)=CLS_H(u_2)$ from precondition, so for any 
 $v'_1\in V(G), v'_2\in V(H)$, which satisfies $CLS_G(v'_1)=CLS_H(v'_2)$, $\phi'(x,y)$ can't distinguish $(u_1,v'_1)$ and $(u_2,v'_2)$. Second, isomorphism tells us $\phi'$ can't distinguish node pairs from the same graph if they share the same $CLS$ values. Combining these two facts, there has to be a 
 subset $S\subseteq \{1,2,3\}$, such that $N_1=\sum_{a\in S}N_G(a)$ and $N_2=\sum_{a\in S}N_H(a)$, where $N_G(a)$ denotes the number of nodes $u$ on $G$ such that $CLS_G(u)=a$, ($N_H(a)$ is defined similarly). 

 It's easy to see that $N_G(1)=N_H(1)=1$, and $N_G(a),N_H(a)>m$ for $a\in \{2,3\}$. Therefore, at least one of $N_1=N_2$ and  $m<min\{N_1,N_2\}$ holds. In neither case $\exists^{\ge N}y\phi'(x,y),N\leq m$ can distuigush $(u_1,v_1)$ and $(u_2,v_2)$.
 \end{proof}

 Note that in the above proof our graph class $\{G(n),H(n)|n\in\mathbb{N}\}$ is actually a \emph{simple} graph class, so we can actually get the following stronger argument.

\begin{corollary}\label{strongeracrfoc}
\rrgnn $\nsubseteq$ $\foc$ on some simple graph class.
\end{corollary}
\section{Proof of \Cref{single-edge}}
\begin{reptheorem}{single-edge} $\foc \subseteq$ \rrgnn on any simple graph class, and $\foc\subsetneq$ \rrgnn on some simple graph class.
\end{reptheorem}

\begin{proof}
We just need to show $\foc \subseteq$ \rrgnn on any simple graph class, and the second part can be just concluded from \Cref{strongeracrfoc}. By \Cref{rsfoceqfoc}, 
$\foc=\rsfoc$, so it suffices to show $\rsfoc \subseteq $ 0/1-GNN. By \Cref{closure}, 0/1-GNN is closed under $\wedge,\vee,\neg$, so we can only focus on formulas in $\rsfoc$ of form $\varphi(x)=\exists^{\ge n}y(\varphi_S(x,y)\wedge\varphi'(y)),S\subseteq \relations$. If we can construct an equivalent 0/1-GNN $\mathcal{A}$ for all formulas of above form, then we can capture all formulas in $\rsfoc$ since other generating rules $\wedge,\vee,\neg$ is closed under 0/1-GNN. In particular, for the setting of \emph{single-edge} graph class, $\varphi$ is meaningful only when $|S|\leq 1$. That's because $|S|>2$ implies that $\varphi$ is just the trivial $\bot$ in any \emph{single-edge} graph class $\sgraph$.

Do induction over quantifier depth $k$ of $\varphi(x)$. In the base case where $k=0$, the result trivially holds since in this situation, the only possible formulas that needs to consider are unary predicates $A(x)$, where $A \in P_1$, which can be captured by the initial one-hot feature. Next, assume our result holds for all formulas with quantifier depth $k$ no more than $m$, it suffices to prove the result when quantifier depth of $\varphi(x)=\exists^{\ge n}y(\varphi_S(x,y)\wedge\varphi'(y))$ is $m+1$. It follows that quantifier depth of $\varphi'(y)$ is no more than $m$.

 By induction, there is a 0/1-GNN model $\mathcal{A}'$ such that $\mathcal{A}'=\varphi'$ on single-edge graph class. To construct $\mathcal{A}$, we only need to append another layer on $\mathcal{A}'$. This layer $L+1$ has dimension $1$, whose parameters $C^{(L+1)},(A^{(L+1)}_j)_{j=1}^K,R^{(L+1)},b^{(L+1)}$ are set as follows:

\begin{enumerate}
    \item  When $|S| = 1$: Suppose $S=\{j\}$, set $A_{j,(1,1)}^{L+1}=1,b^{L+1}=1-n$, where $A_{j,(1,1)}^{L+1}$ denotes the element on the first row and first column of matrix $A_j^{(L+1)}$. Other parameters in this layer are $0$. This construction represents $\feature_v^{(L+1)}=max(0,min((\sum_{u\in\mathcal{N}_{G,j}(v)}\feature_u^{(L)})-(n-1),1))$. Since $\feature_u^{(L)}$ is classification result outputted by $\mathcal{A}'$ which is equivalent to $\varphi'$, $\sum_{u\in\mathcal{N}_{G,j}(v)}\feature_u^{(L)}$ counts the number of $j$-type neighbor $u$ of $v$ that satisfies $\varphi'(u)$. Therefore  $\feature_v^{(L+1)}=1$ if and only if there exists at least $n$ $j$-type neighbors satisfying the condition $\varphi'$, which is exactly what $\varphi(x)$ means.
    
    \item When $|S|=0$: Let 
$K=|P_2|$, for all $j\in [K]$, set $A_{j,(1,1)}^{L+1}=-1$, $R^{(L+1)}_{1,1}=1,b^{L+1}=1-n$ and all other parameters $0$. This construction represents $\feature_v^{(L+1)}=max(0,min((\sum_{u\in V(G)}\feature_u^{(L)})-(\sum_{j=1}^K\sum_{u\in\mathcal{N}_{G,j}(v)}\feature_u^{(L)})-(n-1),1))$. Since we only consider single-edge graph, $(\sum_{u\in V(G)}\feature_u^{(L)})-(\sum_{j=1}^K\sum_{u\in\mathcal{N}_{G,j}(v)}\feature_u^{(L)})$ exactly counts the number of nodes $u$ that satisfies $\varphi'(y)$ and doesn't have any relation with $v$. It's easy to see that $\feature_v^{(L+1)}=1$ iff there exists at least $n$ such nodes $u$, which is exactly what $\varphi(x)$ means.
\end{enumerate}
Hence, we finish the proof for \Cref{single-edge} -- for each $\foc$ formula over the single-edge graph class, we can construct an R$^2$-GNN to capture it. 

\end{proof}
\section{Proof of \Cref{ACR-GNN} and \Cref{complexity}}
\begin{reptheorem}{ACR-GNN}
\rrgnn $\subseteq$ $\foc$ on any bounded graph class, and \rrgnn $\subsetneq$ $\foc$ on some bounded graph class.
\end{reptheorem}
\begin{reptheorem}{complexity}

For any bounded graph class $\bgraph$. Suppose any $G\in\bgraph$ has no more than $N$ nodes, and $\bgraph$ has unary predicate set $P_1$ and relation (binary predicate) set $P_2$. Let $m_1:=|P_1|,m_2:=|P_2|$, then for any node classifier $c$, suppose $c$ can be represented as an \rrgnn with depth (layer number) $L$, then by \Cref{ACR-GNN} there is a $\foc$ classifier $\varphi$ equivalent to $c$ over $\bgraph$. Moreover, the followings hold:

1. The quantifier depth of $\varphi$ is no more than $L$.

2. The size of $\varphi$ (quantified by the number of nodes of $\varphi$'s parse tree) is no more than $2^{2f(L)}$, where $f(L):=2^{2^{2(N+1)f(L-1)}},f(0)=O(2^{2^{2(m_1+m_2))}})$.
\end{reptheorem}
For \Cref{theoremacrffoc}, we just need to show \rrgnn $\subseteq$ $\foc$ on any bounded graph class. The second part can then be shown by the fact that the graph class $\{G_1,G_2\}$ in \Cref{graphcompare} is a bounded graph class but $\foc\nsubseteq$ \rrgnn still holds. In the following proof, we also show how to get the complexity upper bound claimed in \Cref{complexity}. If we want to prove R$^2$-GNN $\subseteq\foc$, it suffices to show that for any R$^2$-GNN $\mathcal{A}$, there exists an equivalent $\foc$ formula $\varphi$ on any bounded graph class $\bgraph$. It implies that for two graphs $G_1,G_2$ and their nodes $a,b$, if they are classified differently by $\mathcal{A}$, there exists some $\foc$ formula $\varphi$ that can distinguish them. Conversly, if $a,b$ can't be distinguished by any $\foc$ formula, then they can't be distinguished by any R$^2$-GNN as well.

% Inspired by it, for a specific model instance $\mathcal{A}$, we can prove the existence of equivalent $\foc$ classifier by following steps:

% 1. Define $\foc(n)$ for any natural number $n$, which is a logical fragment of $\foc$ that has quantifier depth no more than $n$. And prove that for any $n$ and any bounded graph class $\bgraph$, there are only finitely many intrinsically different logical classifiers that can be represented by $\foc(n)$

% 2. Prove that there exists a natural number $N$, such that for any  $G_1,G_2$ and their nodes $a,b$, if $a,b$ can't be distinguished by any $\foc(N)$ classifier, then they can't be distinguished by any ACR-GNN $\mathcal{A}$ as well.

% 3.$\foc(N)$ only contains finitely many different logical classifiers on $\bgraph$, we first define \emph{truth-table} for a set of classifiers as follows:
\begin{definition}\label{truthtabledef}
For a set of classifiers $\Psi=\{\psi_1......\psi_m\}$, a $\Psi$-truth-table $T$ is a 0/1 string of length $m$. $T$ can be seen as a classifier, which classifies a node $v$ to be true if and only if for any $1\leq i\leq m$, the classification result of $\psi_i$ on $v$ equals to $T_i$, where $T_i$ denotes the $i$-th bit of string $T$. We define $\mathcal{T}(\Psi):=\{0,1\}^m$ as the set of all $\Psi$-truth-tables. We have that for any graph $G$ and its node $v$, $v$ satisfies exactly one truth-table $T$.

%Since $\foc(N)$ has only finitely many intrinsically different logical classifier on bounded graph class, $\mathcal{T}(\foc(N))$ denote the set of $\foc(N)$-truth-tables, that is, the set of 0/1 string of length $m$, where $m$ is the number of intrinsically different logical classifiers that can be represented by $\foc(N)$ on bounded graph class. $\bgraph$. Let these different logical classifiers be $\psi_1,\dots,\psi_m$.

%For any $T\in \mathcal{T}(\foc(N))$, $T$ can be seen as a node classifier, which classifies a node $v$ on $G$ as true iff for all $1\leq i\leq m$, $(G,v)\models \psi_i\Leftrightarrow T_i=1$ where $T_i$ denotes the $i$-th bit of $T$.
\end{definition}

\begin{proposition}\label{finite}
Let $\foc(n)$ denote the set of formulas of $\foc$ with quantifier depth no more than $n$. For any  bounded graph class $\bgraph$ and $n$, only finitely many 
 intrinsically different node classifiers on $\bgraph$ can be represented by $\foc(n)$. Furthermore, define $N,m_1,m_2$ as in \Cref{complexity}, the number of intrinsically different $\foc(n)$ node classifiers on $\bgraph$ and their parse tree sizes are all upper bounded by $f(n)$ as defined in \Cref{complexity}.
\end{proposition}
\begin{proof}
Suppose all graphs in $\bgraph$ have no more than $N$ constants, then for any natural number $m>N$, formulas of form $\exists^{\ge m}y (\varphi(x,y))$ are always false. Therefore, it's sufficient only to consider 
$\foc$ logical classifiers with threshold numbers no more than $N$ on $\bgraph$. 

There are only $m_1+m_2$ predicates, and each boolean combination of unary predicates using $\wedge,\vee,\neg$ can be rewritten in the form of Disjunctive Normal Form~(DNF) (\cite{davey2002introduction}). So there are only at most $f(0)=2^{2^{2(m_1+m_2)}}$ intrinsically different formulas in $\foc$ with quantifier depth $0$. Note that $2(m_1+m_2)$ is the number of terms, $2^{2(m_1+m_2)}$ is the number of different truth-table conjunctions on these terms, and $2^{2^{2(m_1+m_2)}}$ is the number of different DNFs on these conjunctions. Each DNF has parse tree of size at most $1+2^{2(m_1+m_2)}(1+2m_1+2m_2)\leq 1000\cdot 2^{2^{2(m_1+m_2)}}$. Therefore, define $f(0)=1000\cdot 2^{2^{2(m_1+m_2)}}=O(2^{2^{2(m_1+m_2)}})$, we know the number of different $\foc$ formulas with quantifier depth $0$ and parse tree size of these formulas can both be upper bounded by $f(0)$.

By induction, suppose there are only $f(k)$ intrinsically different $\foc(k)$ formulas on $\bgraph$. and each meaningful $\foc(k+1)$ formula is generated by the following grammar
\begin{equation}
\varphi_1\wedge\varphi_2,\varphi_1\vee\varphi_2,\neg\varphi_2,\exists^{\ge m}y(\varphi'(x,y)),m\leq N
\end{equation}
where $\varphi_1,\varphi_2$ are $\foc(k+1)$ formulas and $\varphi'$ is $\foc(k)$ formulas.

Given that only the rule $\exists^{\ge m}y(\varphi'(x,y))$ can increase the quantifier depth from $k$ to $k+1$, $m\leq N$, and there are only $f(k)$ intrinsically different $\varphi'(x,y)\in\foc(k)$ on $\bgraph$ by induction. Therefore, there are only $(2N+2)f(k)$ intrinsically different $\foc(k+1)$ formulas of form $\exists^{\ge m}y(\varphi'(x,y)),\exists^{\ge m}x(\varphi'(x,y))$ or in $\foc(k)$ on $\bgraph$. Moreover, their boolean combination using $\wedge,\vee,\neg$ can be always rewritten in the DNF form, So there are also finitely many intrinsically different $\foc(k+1)$ logical classifiers on $\bgraph$. Similarly, we can bound the number of different DNF by $f(k+1)=2^{2^{2(N+1)f(k)}}$, where $2(N+1)f(k)$ is the number of "building blocks" which are sub-formulas with smaller quantifier depth or outermost symbol $\exists$, $2^{2(N+1)f(k)}$ is the number of different conjunctions on these building blocks, and $f(k+1)=2^{2^{2(N+1)f(k)}}$ is the number of different DNFs on these conjunctions. Parse tree size of each of these DNFs is at most $1+2^{2(N+1)f(k)}(1+2(N+1)f(k)(1+f(k)))\leq 2^{2^{2(N+1)f(k)}}=f(k+1)$. The LHS is from the inductive assumption that each $\foc(k)$ formula has a equivalent representation within $f(k)$ parse tree size. The inequality is because we know $f(k)\ge 1000$. Thus, we can upper bound the number of intrinsically different 
 $\foc(k+1)$ formulas on $\bgraph$ and their parse tree size both by $f(k+1)$.
\end{proof}

% Next, we want to prove the second step, that for ACR-GNN model $\mathcal{A}$ with $L$ layers, any two graphs $G_1,G_2$ and  their nodes $a,b$, if $a,b$ can't be distinguished by any $\foc(L)$ classifier, then they can't be distinguished by $\mathcal{A}$ as well.

\begin{lemma}\label{truthtable}
 For any two pairs $(G_1, v_1)$ and $(G_2, v_2)$, where $G_1$ and $G_2$ are two \emph{bounded} graphs from $\bgraph$ and $v_1$ and $v_2$ are two nodes in $G_1$ and $G_2$, respectively. If all logical classifiers in $\foc(L)$ can't distinguish $v_1,v_2$, then any R$^2$-GNN with layer no more than $L$ can't distinguish them as well.
\end{lemma}
\begin{proof}
By one-hot feature initialization function of R$^2$-GNN, $\foc(0)$ can distinguish all different one-hot intial features, so the lemma trivially holds for the base case~($L=0$).

For the inductive step, we suppose \Cref{truthtable} holds for all $L\leq k$, then we can assume $v_1,v_2$ can't be distinguished by $\foc(k+1)$. Let $N=k+1$

$G_1$ and $G_2$ are \emph{bounded} graphs from $\bgraph$, so $\foc(N)$ has finitely many intrinsically different classifiers according to \Cref{finite}. Let $\mathcal{TT}_N(v)$ denote the $\foc(N)$-truth-table satisfied by $v$.  According to \Cref{truthtabledef}, we know that for any $T\in\mathcal{T}(\foc(N))$, there exists a $\foc(N)$ classifier $\varphi_T$ such that for any node $v$ on $G_i$, where $i\in{1, 2}$,  $\mathcal{TT}_N(v)=T\Leftrightarrow (G_{i},v)\models \varphi_T$.

Assume there is an R$^2$-GNN $\mathcal{A}$ that distinguish $v_1,v_2$ with layer $L=k+1$. Let $\widehat{\mathcal{A}}$ denote its first $k$ layers. By update rule of 
R$^2$-GNN illustrated in Equation~\ref{acrgnn}, output of $\mathcal{A}$ on node 
$v$ of graph $G$, $\feature_v^{(k+1)}$ only dependent on the following three things:
\begin{itemize}
\item output of $\widehat{\mathcal{A}}$ on $v$, $\feature_v^{(k)}$
\item multiset of outputs of $\widehat{\mathcal{A}}$ on $r$-type neighbors of $v$ for each $r\in P_2$, $\{\feature_u^{(k)}|u\in\mathcal{N}_{G,r}(v)\}$
\item multiset of outputs of $\widehat{\mathcal{A}}$ on all nodes in the graph, $\{\feature_u^{(k)}|u\in\mathcal{N}_{G,r}(v)\}$
\end{itemize}
By induction, since $v_1,v_2$ can't be distinguished by $\foc(k)$, they has same feature outputted by $\widehat{\mathcal{A}}$. Then  there are two remaining possibilities.
\begin{itemize}
\item  
$\mulsetl\mathcal{TT}_k(u)|u\in\mathcal{N}_{G_1,r}(v_1)\mulsetr\neq \mulsetl\mathcal{TT}_k(u)|u\in\mathcal{N}_{G_2,r}(v_2)\mulsetr$ for some binary predicate $r$. Therefore, there exists a $\foc(k)$-truth-table $T$, such that $v_1,v_2$ have differently many $r$-type neighbors that satisfies $\varphi_T$. Without loss of generality, suppose $v_1,v_2$ have $n_1,n_2 (n_1<n_2)$ such neighbors respectively. we can write a $\foc(k+1)$ formula $\exists^{\ge n_2}y(r(x,y)\wedge \varphi_T(y))$ that distinguishes $v_1$ and $v_2$, which contradicts the precondition that they can't be distinguished by $\foc(k+1)$ classifiers.
\item
$\mulsetl\mathcal{TT}_k(u)|u\in V(G_1)\mulsetr\neq\mulsetl\mathcal{TT}_k(u)|u\in V(G_2)\mulsetr $. Therefore, there exists a $\foc(k)$-truth-table $T$, such that $G_1,G_2$ have differently many nodes that satisfies $\varphi_T$. Without loss of generality, suppose $G_1,G_2$ have $n_1,n_2 (n_1<n_2)$ such nodes respectively. we can write a $\foc(k+1)$ formula $\exists^{\ge n_2}y\varphi_T(y)$ that distinguishes $v_1$ and $v_2$, which contradicts the precondition that they can't be distinguished by $\foc(k+1)$ classifiers.
\end{itemize}
Since all possibilities contradicts the precondition that $v_1,v_2$ can't be distinguished by $\foc(k+1)$, such an $\mathcal{A}$ that distinguishes $v_1,v_2$ doesn't exist.
\end{proof}
We can now gather all of these to prove \Cref{ACR-GNN} and \Cref{complexity}.
\begin{proof}
For any R$^2$-GNN $\mathcal{A}$, suppose it has $L$ layers. For any graph $G\in \bgraph$ and its node $v$, let $\mathcal{TT}_L(v)$ denote the $\foc(L)$-truth-table satisfied by $v$. For any $T\in\mathcal{T}(\foc(L))$, since $\bgraph$ is a bounded graph class, using \Cref{finite}, there exists a $\foc(L)$ classifier $\varphi_T$ such that for any node $v$ in graph $G\in \bgraph$, $\mathcal{TT}_L(v)=T\Leftrightarrow (G,v)\models \varphi_T$. Moreover, by \Cref{finite}, since $T$ is a truth table on at most $f(L)$ formulas, $\varphi_T$ can be written as a conjunction over $f(L)$ literals, which means $\varphi_T$ has parse tree size at most $1+f(L)^2$ since by \Cref{finite}, every formula in $\foc(L)$ is equivalent to some $\foc$ formula with parsee tree size at most $f(L)$.

By \Cref{truthtable}, If two nodes $v_1,v_2$ have same $\foc(L)$-truth-table ($\mathcal{TT}_L(v_1)=\mathcal{TT}_L(v_2)$), they can't be distinguished by $\mathcal{A}$. Let $S$ denote the subset of $\mathcal{T}(\foc(L))$ that satisfies $\mathcal{A}$. By \Cref{finite} and \Cref{truthtabledef}, $\Phi:=\{\varphi_T|T\in S\}$ is a finite set with $|\Phi|\leq 2^{f(L)}$, then disjunction of formulas in $\Phi$, $(\bigvee_{T\in S}\varphi_T)$ is a $\mathcal{FOC}_2$ classifier that equals to $\mathcal{A}$ under bounded graph class $\bgraph$. Furthermore, by the above upper bound of parse tree size of any $\varphi_T$, $(\bigvee_{T\in S}\varphi_T)$ has parse tree size no more than $1+2^{f(L)}(1+f(L)^2)\leq 2^{2f(L)}$, where the inequality is from $f(L)\ge 1000$. We complete the proof.
\end{proof}
\section{Proof of \Cref{theoremacracrf}}
\begin{reptheorem}{theoremacracrf}
\rrgnn $\subseteq$  \rrgnn $\circ F$ on any \emph{universal} graph class $\ggraph$.
\end{reptheorem}

\begin{proof}
Assume that we have a  predicate set $P=P_1\cup P_2$, $K=|P_2|$ and let $P'=P\cup\{primal,aux1,aux2\}$ denote the predicate set after transformation $F$.  
For any R$^2$-GNN $\mathcal{A}$ under $P$, we want to construct another R$^2$-GNN $\mathcal{A}'$ under $P'$, such that for any graph $G$ under $P$ and its node $v$, $v$ has the same feature outputted by $\mathcal{A}(G,v)$ and $\mathcal{A}'(F(G),v)$. Let $L$ denote the layer number of $\mathcal{A}$.

We prove this theorem by induction over the number of layers $L$. In the base~($L=0$), our result trivially holds since the one-hot initialization over $P'$ contains all unary predicate information in $P$. Now suppose the result holds for $L\leq k$, so it suffices to prove it when $L=k+1$.

For the transformed graph $F(G)$, \primal(v) is \emph{true} if and only if $v$ is the node in the original graph $G$. Without loss of generality, if we use one-hot feature initialization on $P'$, we can always keep an additional dimension in the node feature vector $\feature_v$ to show whether \primal(v) is \emph{true}, its value is always $0/1$, in the proof below when we use $\feature$ to denote the feature vectors, we omit this special dimension for simplicity. But keep in mind that this dimension always keeps so we can distinguish original nodes and added nodes.

Recall that an R$^2$-GNN is defined by $\{C^{(i)},(A_j^{(i)})_{i=1}^{K},R^{(i)}\}_{i=1}^{L}$. By induction, let $\widehat{\mathcal{A}}$ denote the first $k$ layers of $\mathcal{A}$, and let $\widehat{\mathcal{A}}'$ denote the R$^2$-GNN equivalent with $\widehat{\mathcal{A}}$ on $F$ transformation such that $\widehat{\mathcal{A}}=\widehat{\mathcal{A}}'\circ F$. We will append three layers to $\widehat{\mathcal{A}}'$ to construct $\mathcal{A}'$ that is equivalent to $\mathcal{A}$. Without loss of generality, we can assume all layers in 
$\mathcal{A}$ have same dimension length $d$. 
  Suppose $L'$ is the layer number of $\widehat{\mathcal{A}}'$, so we will append layer $L'+1,L'+2,L'+3$. for all $l\in\{L'+1,L'+2,L'+3\}$, let $\{C^{a,(l)},C^{p,(l)},(A_j^{*,(l)})_{j=1}^{K},A_{aux1}^{*,(l)},A_{aux2}^{*,(l)},R^{*,(l)}\}$ denote the parameters in $l$-th layer of $\mathcal{A}$. Here, $A_{aux1}^{*,(l)},A_{aux2}^{*,(l)}$ denotes the aggregation function corresponding to two new predicates \iaux,\eaux, added in transformation $F$, and $C^{p,(l)},C^{a,(l)}$ are different combination function that used for primal nodes and non-primal nodes. Note that with the help of the special dimension mentioned above, we can distinguish primal nodes and non-primal nodes. Therefore, It's safe to use different combination functions for these two kinds of nodes. Note that here since we add two predicates \iaux,\eaux, the input for combination function should be in the form $C^p(\feature_0,(\feature_j)_{j=1}^{K},\feature_{aux1},\feature_{aux2},\feature_g)$ where $\feature_0$ is the feature vector of the former layer, and $\feature_j,1\leq j\leq K$ denote the output of aggregation function $A_j^{*,(l)}$, $\feature_{aux1},\feature_{aux2}$ denote the output of aggregation function $A_{aux1}^{*,(l)},A_{aux2}^{*,(l)}$, and $\feature_g$ denotes the feature outputted by global readout function $R^{*,(l)}$. For aggregation function and global readout function, their inputs are denoted by $\fset$, meaning a multiset of feature vector. Note that all aggregation functions and readout functions won't change the feature dimension, only combination functions $C^{p,(l)},C^{a,(l)}$ will transform $d_{l-1}$ dimension features to $d_l$ dimension features.

1).~ layer $L'+1$: input dimension is $d$, output dimension is $d'=Kd$. For feature vector $\textbf{x}$ with length $d'$, let $\textbf{x}^{(i)}, i\in\{1, \dots, K\}$ denote its $i$-th slice in dimension $[(i-1)d+1,id]$. Let $[\feature_1,\dots, \feature_m]$ denote concatenation of $\feature_1, \dots, \feature_m$, and let $[\feature]^n$ denote concatenation of $n$ copies of $\feature$, $\textbf{0}^n$ denote zero vectors of length $n$. parameters for this layer are defined below:
\begin{equation}\label{layer1cp}
C^{p,(L'+1)}(\feature_0,(\feature_j)_{j=1}^{K},\feature_{aux1},\feature_{aux2},\feature_{g})=[\feature_0,\textbf{0}^{d'-d}]
\end{equation}
\begin{equation}\label{layer1ca}
C^{a,(L'+1)}(\feature_0,(\feature_j)_{j=1}^{K},\feature_{aux1},\feature_{aux2},\feature_{g})=[\feature_{aux1}]^{K}
\end{equation}
\begin{equation}\label{layer1oth}
A_{aux1}^{*,(L'+1)}(\fset)=\sum_{\feature\in\fset}\feature
\end{equation}
Other parameters in this layer are set to functions that always output zero-vector.

We can see here that the layer $L'+1$ do the following thing:

For all primal nodes $a$ and its non-primal neighbor $e_{ab}$, pass concatenation of $K$ copies of $\feature_a$ to $\feature_{e_{ab} }$, and remains the feature of primal nodes unchanged.

2).  layer $L'+2$, also has dimension $d'=Kd$, has following parameters.
\begin{equation}
C^{p,(L'+2)}(\feature_0,(\feature_j)_{j=1}^{K},\feature_{aux1},\feature_{aux2},\feature_g)=\feature_0
\end{equation}
\begin{equation}
C^{a,(L'+2)}(\feature_0,(\feature_j)_{j=1}^{K},\feature_{aux1},\feature_{aux2},\feature_g)=\sum_{j=1}^K\feature_j
\end{equation}
\begin{equation}
\forall j\in[1,K],A_j^{*,(L'+2)}(\fset)=[\textbf{0}^{(j-1)d},\sum_{\feature\in\fset}\feature^{(j)},\textbf{0}^{(K-j)d}]
\end{equation}
All other parameters in this layer are set to function that always outputs zero vectors. This layer do the following thing:

For all primal nodes, keep the feature unchanged, for all added node pair $e_{ab},e_{ba}$. Switch their feature, but for all $r_i\in P_2$, if there is no $r_i$ relation between $a,b$, the $i$-th slice of $\feature_{e_{ab}}$ and $\feature_{e_{ba}}$ will be set to $\textbf{0}$.

3). layer $L'+3$, has dimension $d$, and following parameters.
\begin{equation}
C^{p,(L'+3)}(\feature_0,(\feature_j)_{j=1}^{K},\feature_{aux1},\feature_{aux2},\feature_g)=C^{(L)}(\feature_0^{(1)},(\feature_{aux1}^{(j)})_{j=1}^{K},\feature_{g}^{(1)})
\end{equation}
\begin{equation}
R^{*,(L'+3)}(\fset)=[R^{(L)}(\mulsetl\feature_v^{(1)}|\feature_v\in \fset,\textbf{primal}(v)\mulsetr),\textbf{0}^{d'-d}]
\end{equation}
\begin{equation}
A_{aux1}^{*,(L'+3)}(\fset)=[A^{(L)}_1(\mulsetl\feature^{(1)}|\feature\in\fset\mulsetr)......A^{(L)}_K(\mulsetl\feature^{(K)}|\feature\in\fset\mulsetr)]
\end{equation}

Note that $C^{(L)},A_j^{(L)},R^{(L)}$ are all parameters in the last layer of $\mathcal{A}$ mentioned previously. All other parameters in this layer are set to functions that always output zero vectors. We can see that this layer simulates the work of last layer of $\mathcal{A}$ as follows:
\begin{itemize}
\item  For all $1\leq j\leq K$, use the $j$-th slice of feature vector $\feature^{(j)}$ to simulate $A_{j}^{(L)}$ and store results of aggregation function $A_{j}^{(L)}$ on this slice.

\item Global readout trivially emulates what $R^{(L)}$ does, but only reads features for primal nodes. It can be done since we always have a special dimension in feature to say whether it's a primal node.

\item  We just simulate what $C^{(L)}$ does on primal nodes. For $1\leq j\leq K$ The type $r_j$ aggregation result (output of $A_j^{(L)}$) used for input of $C^{(L)}$ is exactly $j$-th slice of return value of $A_{aux1}^{*,(L'+3)}$. 
\end{itemize}
By construction above, $\mathcal{A}'$ is a desired model that have the same output as $\mathcal{A}$.

\end{proof}
\section{Proof of \Cref{theoremfocacrf}}
\begin{reptheorem}{theoremfocacrf}
$\foc \subseteq$ \rrgnn $\circ F$ on any  \emph{universal} graph class $\ggraph$.
\end{reptheorem}
\begin{proof}
For any $\foc$ classifier $\varphi$ under predicate set $P$, we want to construct a 0/1-GNN $\mathcal{A}$ on $P'=P\cup\{primal,aux1,aux2\}$ equivalent to $\varphi$ with \emph{graph transformation} $F$.

Recall that $\foc = \rsfoc$ shown in ~\Cref{rsfoceqfoc} and 0/1-GNNs $\subseteq$ R$^2$-GNNs, it suffices to  prove that 0/1-GNN$\circ F$ capture $\rsfoc$. By \Cref{closure}, since $\wedge,\vee,\neg$ are closed under 0/1-GNN it suffices to show that when $\varphi$ is in the form $\exists^{\ge n}\big(
\varphi_S(x,y)\wedge\varphi'(y)\big ),S\subseteq P_2$, we can capture it.

We prove by induction over quantifier depth $m$ of $\varphi$. Since $0$-depth formulas are only about unary predicate that 
can be extracted from one-hot initial feature, our theorem trivially holds for $m=0$. Now, we assume it also holds for $m\leq k$, it suffices to prove the case when $m=k+1$. Then there are two possibilities:

1. When $S\neq \emptyset$:

Consider the following logical classifier under $P'$:
\begin{equation}
\widehat{\varphi}_S(x):=\Bigl(\bigwedge_{r\in S}\exists xr(x,y)\Bigl)\wedge\Bigl(\bigwedge_{r\notin S}\neg \exists xr(x,y)\Bigl)
\end{equation}
$\widehat{\varphi}_S(x)$ restricts that for any $r \in P'$, $x$ has $r$-type neighbor if and only if $r\in S$. Review the definition of transformation $F$, we know that for any added node $e_{ab}$, $(F(G),e_{ab})\models \widehat{\varphi}_S$ if and only if $(G,a,b)\models\varphi_S(a,b)$, where $\varphi_S(x,y)$ is the \emph{relation-specification} formula defined in \Cref{rsfoc} That is to say for any $r_i,1\leq i\leq K$, there is relation $r_i$ between $a,b$ if and only if $i\in S$.

Now consider the following formula:
\begin{equation}
\widehat{\varphi}:=\exists^{\ge n}y\biggl(\iaux(x,y)\wedge\widehat{\varphi}_S(y)\wedge\Bigl(\exists x\bigl(\eaux(x,y)\wedge(\exists y(\iaux(x,y)\wedge\varphi'(y)))\bigl)\Bigl) \biggl)
\end{equation}
For any graph $G$ and its node $v$, it's easy to see that $(G,v)\models \varphi\Leftrightarrow (F(G),v)\models \widehat{\varphi}$. Therefore we only need to capture $\widehat{\varphi}$ by 0/1-GNN on every primal node of transformed graphs. By induction, since quantifier depth of $\varphi'(y)$ is no more than $k$, we know $\varphi'(y)$ is in 0/1-GNN. $\widehat{\varphi}$ is generated from $\varphi'(y)$ using rules $\wedge$ and $\exists y\bigl(r(x,y)\wedge \varphi'(y)\bigl)$. By \Cref{closure}, $\wedge$ is closed under 0/1-GNN. For $\exists y\bigl(r(x,y)\wedge \varphi'(y)\bigl)$, we find that the construction needed is the same as construction for single-element $S$ on single-edge graph class $\sgraph$ used in \Cref{single-edge}. Therefore, since we can manage these two rules, we can also finish the construction for $\widehat{\varphi}$, which is equivalent to $\varphi$ on primal nodes of transformed graph.

2. When $S=\emptyset$

First, consider the following two logical classifiers:

\begin{equation}
\bar{\varphi}(x):=\Bigl(\primal(x)\wedge \varphi'(x)\Bigl)
\end{equation}
$\bar{\varphi}$ says a node is primal, and satisfies $\varphi'(x)$. Since $\varphi'(x)$ has quantifier depth no more than $k$, and $\wedge$ is closed under 0/1-GNN. There is a 0/1-GNN $\mathcal{A}_1$ equivalent to $\bar{\varphi}$ on transformed graph. Then, consider the following formula.
\begin{equation}
\tilde{\varphi}(x):=\exists y\bigl(\eaux(x,y)\wedge(\exists x,\iaux1(x,y)\wedge\varphi'(x))\bigl)
\end{equation}

$\tilde{\varphi}(x)$ evaluates on added nodes $e_{ab}$ on transformed graph, $e_{ab}$ satisfies it iff $b$ satisfies $\varphi'$

Now for a graph $G$ and its node $v$, define $n_1$ as the number of nodes on $F(G)$ that satisfies $\bar{\varphi}$, and define $n_2$ as the number of \iaux-type neighbors of $v$ on $F(G)$ that satisfies $\tilde{\varphi}$. Since $\varphi(x)=\exists^{\ge n}y(\varphi_{\emptyset}(x,y)\wedge\varphi'(y))$ It's easy to see that $(G,v)\models \varphi$ if and only if $n_1-n_2\ge n$.

Formally speaking, for a node set $S$, let $|S|$ denote number of nodes in $S$, we define the following classifier $c$ such that for any graph $G$ and its node $a$, $c(F(G),a)=1\Leftrightarrow (G,a)\models \varphi$ 
\begin{equation}\label{classifierc}
c(F(G),a)=1\Leftrightarrow |\{v|v\in V(F(G)), (F(G),v)\models \bar{\varphi}\}|-|\{v|v\in \mathcal{N}_{F(G),\iaux1}(v), (F(G),v)\models \tilde{\varphi}\}|\ge n
\end{equation}

So how to construct a model $\mathcal{A}$ to capture classifier $c$? First, by induction $\bar{\varphi},\tilde{\varphi}$ are all formulas with quantifier depth no more than $k$ so by previous argument there are 0/1-GNN models $\bar{\mathcal{A}},\tilde{\mathcal{A}}$ that capture them respectively. Then  we can use feature concatenation technic introduced in \Cref{parconcat} to construct a model $\widehat{A}$ based on $\bar{\mathcal{A}},\tilde{\mathcal{A}}$, such that $\widehat{A}$ has two-dimensional output, whose first and second dimensions have the same output as $\bar{\mathcal{A}},\tilde{\mathcal{A}}$ respectively.

Then, suppose $\widehat{\mathcal{A}}$ has $L$ layers, The only thing we need to do is to append a new layer $L+1$ to $\widehat{\mathcal{A}}$, it has output dimension $1$. parameters of it are $\{C^{(L+1)},(A_j^{(L+1)})_{j=1}^K,A_{aux1}^{(L+1)},A_{aux2}^{(L+1)},R^{(L+1)}\}$ as defined in \Cref{01acrgnn}. The parameter settings are as follows:

$\textbf{R}^{(L+1)}_{1,1}=1,\textbf{A}_{aux1,(1,2)}^{(L+1)}=-1,\textbf{b}^{(L+1)}_1=1-n$. Other parameters are set to $0$, where $\textbf{A}_{aux1,(1,2)}^{(L+1)}$ denotes the value in the first row and second column of $\textbf{A}_{aux1}^{(L+1)}$.

In this construction, we have 

$\feature_v^{(L+1)}=max(0,min(1,\sum_{u\in V(F(G))}\feature_{u,1}^{(L)}-\sum_{u\in\mathcal{N}_{F(G),aux1}(v)}\feature_{u,2}^{(L)}-(n-1)))$, which has exactly the same output as classifier $c$ defined above in \Cref{classifierc}. Therefore, $\mathcal{A}$ is a desired model.
\end{proof}
\section{Proof of \Cref{theoremacrffoc}}
\begin{reptheorem}{theoremacrffoc}
\rrgnn $\circ F \subseteq \foc$ on  any \emph{bounded} graph class $\bgraph$.
\end{reptheorem}

Before we go into theorem itself, we first introduce \Cref{enum} that will be used in following proof.
\begin{lemma}\label{enum}
Let $\varphi(x,y)$ denote a $\foc$ formula with two free variables, for any natural number $n$, 
the following sentence can be captured by $\foc$:

\textbf{There exists no less than $n$ ordered node pairs $(a,b)$ such that $(G,a,b)\models\varphi$}.

Let $c$ denote the graph classifier such that $c(G)=1$ iff $G$ satisfies the sentence above.
\end{lemma}
\begin{proof}
The basic intuition is to define $m_i,1\leq i<n$ as the number of nodes $a$, such that there are \textbf{exactly} $i$ nodes $b$ that $\varphi(a,b)$ is \emph{true}. Specially, we define $m_n$ as the number of nodes $a$, such that there are \textbf{at least} $n$ nodes $b$ that $\varphi(a,b)$ is \emph{true}. Since $\sum_{i=1}^nim_i$ exactly counts the number of valid ordered pairs when $m_n=0$, and it guarantees the existence of at least $n$ 
 valid ordered pairs when $m_n>0$. It's not hard to see that for any graph $G$, $c(G)=1\Leftrightarrow \sum_{i=1}^nim_i\ge n$. Futhermore, fix a valid sequence $(m_1......m_n)$ such that $\sum_{i=1}^nim_i\ge n$, there has to be another sequence $(k_1......k_n)$ such that $n\leq\sum_{i=1}^nik_i\leq 2n$ and $k_i\leq m_i$ for all $1\leq i\leq n$. Therefore, We can enumerate all possibilities of valid $(k_1......k_n)$, and for each valid $(k_1......k_n)$ sequence, we judge whether there are \textbf{at least} $k_i$ such nodes $a$ for every $1\leq i\leq n$. 

Formally, $\varphi_i(x):=\exists^{[i]}y\varphi(x,y)$ can judge whether a node $a$ has exactly $i$ partners $b$ such that $\varphi(a,b)=1$, where $\exists^{[i]}y\varphi(x,y)$ denotes "there are exactly $i$ such nodes $y$" which is the abbreviation of formula $(\exists^{\ge i}y\varphi(x,y))\wedge(\neg\exists^{\ge i+1}y\varphi(x,y))$. The $\foc$ formula equivalent to our desired sentence $c$ is as follows:
\begin{equation}
\bigvee_{\sum_{i=1}^n n\leq ik_i\leq 2n}\biggl(\bigwedge_{i=1}^{n-1}\exists^{\ge k_i}x\Bigl(\exists^{[i]}y\varphi(x,y)\Bigl)\biggl)\wedge\biggl(\exists^{\ge k_n}x\Bigl(\exists^{\ge n}y\varphi(x,y)\Bigl)\biggl)
\end{equation}
This $\foc$ formula is equivalent to our desired classifier $c$. 
\end{proof}
% \begin{theorem}
% \Cref{theoremacrffoc} ACR-GNN $\circ F\subseteq \mathcal{FOC}_2$ on any bounded graph class $\bgraph$.
% \end{theorem}
With the \Cref{enum}, we now start to prove \Cref{theoremacrffoc}.
\begin{proof}
By \Cref{ACR-GNN}, it follows that R$^2$-GNNs $\circ F\subseteq\foc\circ F$. Therefore it suffices to show $\foc\circ F\subseteq \foc$.

By \Cref{rsfoceqfoc}, it suffices to show $\rsfoc\circ F\subseteq \foc$. Since $\wedge,\vee,\neg$ are common rules. We only need to show for any $\rsfoc$ formula of form $\varphi(x):=\exists^{\ge n}y(\varphi_{S}(x,y)\wedge\varphi'(y))$ under transformed predicate set $P'=P\cup\{aux1,aux2,primal\}$, there exists an $\foc$ formula $\varphi^1$ such that for any graph $G$ under $P$ and its node $v$, $(G,v)\models \varphi^1\Leftrightarrow (F(G),v)\models \varphi$.

In order to show this, we consider a stronger result:

For any such formula $\varphi$, including the existence of 
valid $\varphi^1$, we claim there also exists an $\foc$ formula  $\varphi^2$ with 
 two free variables such that the following holds: for any graph $G$ under $P$ and its added node $e_{ab}$ on $F(G)$, $(G,a,b)\models \varphi^2\Leftrightarrow (F(G),e_{ab})\models \varphi$. Call $\varphi^1,\varphi^2$ as first/second discriminant of $\varphi$.

 Now we need to prove the existence of $\varphi^1$ and $\varphi^2$.

We prove by induction over quantifier depth $m$ of $\varphi$, Since we only add a single unary predicate \primal\  in $P'$, any $\varphi(x)$ with quantifier depth $0$ can be rewritten as $(primal(x)\wedge\varphi^1(x))\vee(\neg primal(x)\wedge\varphi^2(x))$, where $\varphi^1(x),\varphi^2(x)$ are two formulas that only contain  predicates in $P$. Therefore, $\varphi^1$ can be naturally seen as the first discriminant of $\varphi$. Moreover, since $\varphi^2(x)$ always evaluates on non-primal nodes, it is equivalent to $\bot$ or $\top$ under $\neg primal(x)$ constraint. Therefore, the corresponding $\bot$ or $\top$ can be seen as the second discriminant, so our theorem trivially holds for $m=0$. Now assume it holds for $m\leq k$, we can assume quantifier depth of $\varphi=\exists^{\ge n}y(\varphi_{S}(x,y)\wedge\varphi'(y))$ is $m=k+1$.

 Consider the construction rules of transformation $F$, for any two primal nodes in $F(G)$, there is no relation between them, for a primal node $a$ and an added node $e_{ab}$, there is exactly a single relation of type \iaux\   between them. For a pair of added nodes $e_{ab},e_{ba}$, there are a bunch of relations from the original graph $G$ and an additional \eaux\  relation between them. Therefore, it suffices to 
 only consider three possible kinds of $S\subseteq P_2\cup\{aux1,aux2\}$ according to three cases mentiond above. Then, we will construct first/second determinants for each of these three cases. Since $\varphi'(y)$ has quantifier depth no more than $k$, by induction let $\widehat{\varphi}^1,\widehat{\varphi}^2$ be first/second discriminants of $\varphi'$ by induction.

 1. $S=\{\textbf{aux1}\}$:
 
  for primal node $a$, $\varphi(a)$ means the following: there exists at least $n$ nodes $b$, such that there is some relation between $a,b$ on $G$ and the added node $e_{ab}$ on $F(G)$ satisfies $\varphi'$. Therefore, the first determinant of $\varphi$ can be defined as following:
 \begin{equation}
 \varphi^1(x):=\exists^{\ge n}y,\Bigl(\bigvee_{r\in P_2}r(x,y)\Bigl)\wedge\widehat{\varphi}^2(x,y)
 \end{equation}
 for added nodes $e_{ab}$ on $F(G)$, $\varphi(e_{ab})$ means $a$ satisfies $\varphi'$, so the second determinant of $\varphi$ is the following:
 \begin{equation}
 n=1:\varphi^2(x,y):=\widehat{\varphi}^1(x),\ \ n>1: \varphi^2(x,y):=\bot
 \end{equation}
 2.$S=\{\textbf{aux2}\}\cup T,T\subseteq P_2,T\neq \emptyset$

 primal nodes don't have \eaux\  neighbors, so first determinant is trivially \emph{false}.
 \begin{equation}
 \varphi^1(x):=\bot
 \end{equation}
 For added node $e_{ab}$, $e_{ab}$ satisfies $\varphi$ iff there are exactly relations between $a,b$ of types in 
$T$, and $e_{ba}$ satisfies $\varphi'$. Therefore the second determinant is as follows, where $\varphi_T(x,y)$ is the \emph{relation-specification} formula under 
$P$ introduced in \Cref{rsfoc} 
 \begin{equation}
 n=1: \varphi^2(x,y):=\varphi_T(x,y)\wedge\widehat{\varphi}^2(y,x),\ n>1: \varphi^2(x,y):=\bot
 \end{equation}
 
 3. $S=\emptyset$

For a subset $S\subseteq P_2\cup\{aux1,aux2\}$, let $\varphi_S(x,y)$ denote the \emph{relation-specification} formula under $P_2\cup\{aux1,aux2\}$ defined in \Cref{rsfoc}.

 Since we consider on bounded graph class $\bgraph$, node number is bounded by a natural number $N$. For any node $a$ on $F(G)$, let $m$ denote the number of nodes $b$ on $F(G)$ such that $\varphi'(b)=1$, let $m_0$ denote the number of nodes $b$ on $F(G)$ such that $\varphi'(b)=1$ and there is a single relation \iaux, between $(a,b)$ on $F(G)$, (That is equivalent to $\varphi_{\{aux1\}}(a,b)=1$). For any $T\subseteq P_2$, let $m_T$ denote the number of nodes $b$ on $F(G)$ such that $\varphi'(b)=1$ and $a,b$ has exactly relations of types in $T\cup\{aux2\}$ on $F(G)$, (That is equivalent to $\varphi_{T\cup\{aux2\}}(a,b)=1$).

 Note that the number of nodes $b$ on $F(G)$ such that $a,b$ don't have any relation, (That is equivalent to $\varphi_{\emptyset}(a,b)=1$)  and  $\varphi'(b)=1$  equals to $m-m_0-\sum_{T\subseteq P_2}m_T$. Therefore, for any transformed graph $F(G)$ and its node $v$, $(F(G),v)\models\varphi\Leftrightarrow m-m_0-\sum_{T\subseteq P_2}m_T\ge n$. Since $|V(G)|\leq N$ for all $G$ in bounded graph class $\bgraph$, transformed graph $F(G)$ has  node number no more than $N^2$. Therefore, we can enumerate all possibilities of $m,m_0,m_T\leq N^2,T\subset P_2$ such that the above inequality holds, and for each possibility, we judge whehter there exists exactly such number of nodes for each corresponding parameter. Formally speaking, $\varphi$ can be rewritten as the following form: 
 \begin{equation}\label{phimm0}
 \tilde{\varphi}_{m,m_0}(x):=\bigl(\exists^{[m]}y\varphi'(y)\bigl)\wedge(\exists^{[m_0]}y(\varphi_{\{aux1\}}(x,y)\wedge\varphi'(y)))\bigl)
 \end{equation}
 \begin{equation}\label{phim}
 \varphi(x)\equiv\bigvee_{m-m_0-\sum_{T\subseteq P_2}\ge n,0\leq m,m_0,m_T\leq N^2}\biggl(\tilde{\varphi}_{m,m_0}(x)\wedge\bigl(\bigwedge_{T\subseteq P_2}\exists^{[m_T]}y,(\varphi_{T\cup\{aux2\}}(x,y)\wedge \varphi'(y))\bigl)\biggl)
 \end{equation}
 where $\exists^{[m]}y$ denotes there are exactly $m$ nodes $y$.
 
 Since first/second determinant can be constructed trivially under combination of $\wedge,\vee,\neg$, and we've shown how to construct determinants for formulas of form $\exists^{\ge n}y(\varphi_S(x,y)\wedge\varphi'(y))$ when $S=\{aux1\}$ and $S=\{aux2\}\cup T,T\subseteq P_2$ in the previous two cases. Therefore, in \Cref{phimm0} and \Cref{phim}, the only left part is the formula of form $\exists^{[m]}y\varphi'(y)$. The only remaining work is to show how to construct first/second determinants for formula in form $\varphi(x):=\exists^{\ge n} y\varphi'(y)$.

 Let $m_1$ denote the number of primal nodes $y$ that satisfies $\varphi'(y)$ and let $m_2$ denote the number of non-primal nodes $y$ that satisfies $\varphi'(y)$. It's not hard to see that for any node $v$ on $F(G)$, $(F(G),v)\models \varphi\Leftrightarrow m_1+m_2\ge n$. Therefore, $\varphi(x)=\exists^{\ge n} y\varphi'(y)$ that evaluates on $F(G)$ is equivalent to the following sentence that evaluates on $G$: \emph{``There exists two natural numbers $m_1,m_2$ such that the following conditions  hold: \textbf{1.}  $m_1+m_2=n$. \textbf{2.}   There are at least $m_1$ nodes $b$ on $G$ that satisfies $\widehat{\varphi}^1$, (equivalent to $(F(G),b)\models\varphi'$). 
 \textbf{3.} 
 There are at least $m_2$ ordered node pairs $a,b$ on $G$ such that $a,b$ has some relation and $(G,a,b)\models\widehat{\varphi}^2$, (equivalent to $(F(G),e_{ab})\models \varphi'$)."}

 Formally speaking, rewrite the sentence above as formula under $P$, we get the following construction for first/second determinants of $\varphi$.
 \begin{equation}
\varphi^1(x)=\varphi^2(x,y)=\bigvee_{m_1+m_2=n}\Bigl((\exists^{\ge m_1}y,\widehat{\varphi}^1(y))\wedge\overline{\varphi}_{m_2}\Bigl)
 \end{equation}
 where $\overline{\varphi}_{m_2}$ is the $\foc$ formula that expresses \emph{``There exists at least $m_2$ ordered node pairs $(a,b)$ such that $(G,a,b)\models \widehat{\varphi}^2(x,y)\wedge(\bigvee_{r\in P_2}r(x,y))$"}. We've shown the existence of $\overline{\varphi_{m_2}}$ in \Cref{enum}
\end{proof}
\section{Proof of \Cref{tgnnrelation}}
\begin{reptheorem}{tgnnrelation}
time-and-graph $\subsetneq$  R$^2$-TGNN $\circ F^T=$ time-then-graph.
\end{reptheorem}

% \emph{time-and-graph} and \emph{time-then-graph} presented in \cite{gao2022equivalence} are as follows:

For a graph $G$ with $n$ nodes, let $\mathbb{H}^V\in \mathbb{R}^{n\times d_v}$ denote node feature matrix, and $\mathbb{H}^E\in \mathbb{R}^{n\times n\times d_e}$ denote edge feature matrix, where $\mathbb{H}^E_{ij}$ denote the edge feature vector from $i$ to $j$.

First we need to define the GNN used in their frameworks. Note that for the comparison fairness, we add the the global readout to the node feature update as we do in R$^2$-GNNs. It recursively calculates the feature vector $\mathbb{H}^{V,(l)}_i $ of the node i at each layer $1\leq l\leq L$ as follows:

\begin{equation}\label{gaognn}
\mathbb{H}^{V,(l)}_i=u^{(l)}\Bigl(g^{(l)}(\mulsetl(\mathbb{H}_i^{V,(l-1)},\mathbb{H}_j^{V,(l-1)},\mathbb{H}_{ij}^E) \mid j\in \mathcal{N}(i)\mulsetr), r^{(l)}(\mulsetl\mathbb{H}_j^{V,(l-1)}|j\in V\mulsetr)\Bigl)
\end{equation}

where $\mathcal{N}(i)$ denotes the set of all nodes that adjacent to $i$, and $u^{(l)},g^{(l)},r^{(l)}$ are learnable functions. Note that here the GNN framework is a little different from the general definition defined in \Cref{acrgnn}. However, this framework  is hard to fully implement and many previous works implementing \emph{time-and-graph} or \emph{time-then-graph}  \cite{gao2022equivalence} (\cite{https://doi.org/10.48550/arxiv.1905.03994}, \cite{https://doi.org/10.48550/arxiv.1612.07659}, \cite{https://doi.org/10.48550/arxiv.1812.04206}, \cite{Manessi_2020}, \cite{https://doi.org/10.48550/arxiv.1812.09430},\cite{https://doi.org/10.48550/arxiv.2006.10637}) don't reach the expressiveness of \Cref{gaognn}. This definition is more for the theoretical analysis. In contrast, our definition for GNN in \Cref{acgnn} and \Cref{acrgnn} is more practical since it is fully captured by a bunch of commonly used models such as \cite{schlichtkrull2018modeling}. For notation simplicity, for a GNN $\mathcal{A}$, let $\mathbb{H}^{V,(L)}=\mathcal{A}(\mathbb{H}^V,\mathbb{H}^E)$ denote the node feature outputted by $\mathcal{A}$ using $\mathbb{H}^V,\mathbb{H}^E$ as initial features.

\begin{proposition}\label{gao22neq}
(\cite{gao2022equivalence}):time-and-graph $\subsetneq$ time-then-grahp
\end{proposition}

The above proposition is from \textbf{Theorem 1} of \cite{gao2022equivalence}. Therefore, in order to complete the proof of \Cref{tgnnrelation}, we only need to prove R$^2$-TGNN $\circ F^T=$ \emph{time-then-graph}.

Let $G=\{G_1, \dots, G_T\}$ denote a temporal knowledge graph, and $\mathbb{A}^t\in \mathbb{R}^{n\times |P_1|}, \mathbb{E}^t\in \mathbb{R}^{n\times n\times |P_2|}, 1\leq t\leq T$ denonte one-hot encoding feature of unary facts and binary facts on timestamp $t$, where $P_1,P_2$ are unary and binary predicate sets.

The updating rule of a \emph{time-then-graph} model can be generalized as follows:

\begin{equation}\label{ttg1}
\forall i\in V, \ \mathbb{H}_i^V=\textbf{RNN}([\mathbb{A}^1_i......\mathbb{A}^T_i])
\end{equation}
\begin{equation}\label{ttg2}
\forall i,j\in V, \ \mathbb{H}_{i,j}^E=\textbf{RNN}([\mathbb{E}^1_{i,j}......\mathbb{E}^T_{i,j}])
\end{equation}
\begin{equation}\label{ttg3}
\ \fset:=\mathcal{A}(\mathbb{H}^V,\mathbb{H}^E)
\end{equation}
where $\mathcal{A}$ is a GNN defined above, \textbf{RNN} is an arbitrary Recurrent Neural Network. $\fset\in \mathbb{R}^{n\times d}$ is the final node feature output of \emph{time-then-graph}.

First we need to prove \emph{time-then-graph} $\subseteq$ R$^2$-TGNN$\circ F^T$. That is, for any \emph{time-then-graph} model, we want to construct an equivalent R$^2$-TGNN $\mathcal{A}'$ to capture it on transformed graph. We can use nodes added after transformation to store the edge feature $\mathbb{H}^E$, and use primal nodes to store the node feature $\mathbb{H}^V$. By simulating \textbf{RNN} through choosing specific functions in R$^2$-TGNN, we can easily construct a R$^2$-TGNN $\mathcal{A}'$ such that for any node $i$, and any node pair $i,j$ with at least one edge in history, $\feature_i=\mathbb{H}_i^V$ and $\feature_{e_{ij}}=\mathbb{H}_{i,j}^E$ hold, where $\feature_i$ and $\feature_{e_{ij}}$ are features of corresponding primal node $i$ and added node $e_{ij}$ outputted by $\mathcal{A}'$. 

Note that $\mathcal{A}'$ is a R$^2$-TGNN, it can be represented as $\mathcal{A}'_1......\mathcal{A}'_T$, where each $\mathcal{A}'_t, 1\leq t\leq T$ is a R$^2$-GNN. $\mathcal{A}'$ has simulated work of \textbf{RNN}, so the remaining work is to simulate $\mathcal{A}(\mathbb{H}^V,\mathbb{H}^E)$. We do the simulation over induction on layer number $L$ of $\mathcal{A}$.

When $L=0$, output of $\mathcal{A}$ is exactly $\mathbb{H}^V$, which has been simulated by $\mathcal{A}'$ above.

Suppose $L=k+1$, let $\tilde{\mathcal{A}}$ denote R$^2$-GNN extracted from $\mathcal{A}$ but without the last layer $k+1$. By induction, we can construct a  R$^2$-TGNN $\tilde{\mathcal{A}}'$ that simulates $\tilde{\mathcal{A}}(\mathbb{H}^V,\mathbb{H}^E)$. Then we need to append three layers to $\tilde{\mathcal{A}}'$ to simulate the last layer of $\mathcal{A}$.

Let $u^{(L)},g^{(L)},r^{(L)}$ denote parameters of the last layer of $\mathcal{A}$.  Using notations in \Cref{acrgnn}, let $\{C^{(l)},(A_j^{(l)})_{j=1}^{|P_2|},A_{aux1}^{(l)},A_{\eaux}^{(l)},R^{(l)}\}_{l=1}^3$ denote parameters of the three layers appended to $\tilde{\mathcal{A}}'_T$. They are defined as follows:

First, we can choose specific function in the first two added layers, such that the following holds:

\textbf{1.} For any added node $e_{ij}$, feature outputted by the new model is $\feature_{e_{ij}}^{(2)}=[\mathbb{H}_{ij}^E,\feature'_i,\feature'_j]$, where $\feature^{(2)}$ denotes the feature outputted by the second added layer, and $\feature'_i,\feature'_j$ are node features of $i,j$ outputted by $\tilde{\mathcal{A}}'$. For a feature $\feature$ of added node of this form, we define $\feature_{0},\feature_1,\feature_2$ as corresponding feature slices where $\mathbb{H}_{ij}^E,\feature'_i,\feature'_j$ 
 have been stored. 
 
 \textbf{2.} For any primal node, its feature $\feature$ 
 only stores $\feature'_i$ in $\feature_1$, and $\feature_0,\feature_2$ are all slices of dummy bits.

 Let $\fset$ be a multiset of features that represents function input. For the last added layer, we can choose specific functions as follows:

\begin{equation}
R^{(3)}(\fset):=r^{(L)}(\mulsetl \feature_1|\feature\in\fset,\textbf{primal}(\feature)\mulsetr)
\end{equation}
\begin{equation}
A_{aux1}^{(3)}(\fset):=g^{(L)}(\mulsetl( \feature_1,\feature_2,\feature_0)|\feature\in\fset\mulsetr)
\end{equation}
\begin{equation}
C^{(3)}(\feature_{aux1},\feature_g):=u^{(L)}(\feature_{aux1},\feature_g)
\end{equation}

where $\feature_{aux1},\feature_g $ are outputs of $R^{(3)}$ and $A_{aux1}^{(3)}$, and all useless inputs of $C^{(3)}$ are omitted. Comparing this construction with \Cref{gaognn}. It's east to see that after the last layer appended, we can construct an equivalent R$^2$-TGNN $\mathcal{A}'$ that  captures $\mathcal{A}$ on transformed graph. By inductive argument, we prove \emph{time-then-graph} $\subseteq$ R $^2$-TGNN $\circ F^T$.

Then we need to show R$^2$-TGNN $\circ F^T\subseteq$ \emph{time-then-graph}.

In \Cref{ttransformation}, we will prove R$^2$-TGNN $\circ F^T=$ R$^2$-GNN $\circ F\circ H$. Its proof doesn't dependent on \Cref{tgnnrelation}, so let's assume it's true for now. Then, instead of proving R$^2$-TGNN $\circ F^T$, it's sufficient to show R$^2$-GNN $\circ F\circ H\subseteq$ \emph{time-then-graph}.

Let $P^T_1, P^T_2$ denote the set of temporalized unary and binary predicate sets defined in \Cref{temporalize}. Based on \emph{most expressive ability} of Recurrent Neural Networks shown in \cite{10.1145/130385.130432}, we can get a \emph{most expressive representation} for unary and binary fact sequences through \textbf{RNN}. A \emph{most expressive} RNN representation function is
always injective, thus there exists a decoder function translating most-expressive representations
back to raw sequences. Therefore, we are able to find an appropriate \textbf{RNN} such that its output features $\mathbb{H}^V,\mathbb{H}^E$  in \Cref{ttg1}, \Cref{ttg2} contain all information needed to reconstruct all temporalized unary and binary facts related to the corresponding nodes.

For any R$^2$-GNN $\mathcal{A}$ on transformed collpsed temporal knowledge graph, we want to construct an equivalent \emph{time-then-graph} model $\{\textbf{RNN},\mathcal{A}'\}$ to capture $\mathcal{A}$. In order to show the existence of the \emph{time-then-graph} model, we will do an inductive construction over layer number $L$ of $\mathcal{A}$. Here in order to build inductive argument, we will consider a following stronger result and aim to prove it: In additional to the existence of $\mathcal{A}'$, we claim there also exists a function $f_{\mathcal{A}}$ with the following property: For any two nodes $a,b$ with at least one edge,  $f_{\mathcal{A}}(\feature_a',\feature_b',\mathbb{H}_{ab}^{E})=\feature_{e_{ab}}$, where  $\feature_a',\feature_b',\mathbb{H}_{ab}^{E}$ are  features of $a$, $b$ and edge information between $a,b$ outputted by 
$\mathcal{A}'$, and $\feature_{e_{ab}}$ is the feature of added node $e_{ab}$ outputted by $\mathcal{A}\circ F\circ H$. It suffices to show that there exists such function $f_{\mathcal{A}}$ as well as a \emph{time-then-graph} model $\{\textbf{RNN},\mathcal{A}'\}$ such that the following conditions hold:

For any graph $G$ and its node $a,b\in V(G)$, 

1. $\mathbb{H}_a^{V,(l)}=[\feature_a,Enc(\mulsetl\feature_{e_{aj}}|j\in \mathcal{N}(a)\mulsetr)]$.

2.If there is at least one edge between $a,b$ in history, $f_{\mathcal{A}}(\mathbb{H}_a^{V,(l)},\mathbb{H}_b^{V,(l)},\mathbb{H}_{ab}^E)=\feature_{e_{ab}}$. Otherwise, $f_{\mathcal{A}}(\mathbb{H}_a^{V,(l)},\mathbb{H}_b^{V,(l)},\mathbb{H}_{ab}^E)=\textbf{0}$

where $\mathbb{H}_a^{V,(l)},\mathbb{H}_b^{V,(l)}$ are node features outputted by $\mathcal{A}'$, while $\feature_a,\feature_{e_{ab}}$ are node features outputted by $\mathcal{A}$ on transformed collpased graph. $Enc(\fset)$ is some injective encoding that stores all information of multiset $\fset$. For a node feature $\mathbb{H}_a^{V,(l)}$ of above form, let $\mathbb{H}_{a,0}^{V,(l)}:=\feature_a, \mathbb{H}_{a,1}^{V,(l)}=Enc(\mulsetl\feature_{e_{aj}}|j\in \mathcal{N}(a)\mulsetr)$ denote two slices that store independent information in different positions.

For the base case $L=0$. the node feature only depends on temporalized unary facts related to the corresponding node. Since by \textbf{RNN} we can use \emph{most expressiveness representation} to capture all unary facts. A specific \textbf{RNN} already captures $\mathcal{A}$ when $L=0$. Moreover, there is no added node $e_{ab}$ that relates to any unary fact, so a constant function 
 already satisfies the condition of $f_{\mathcal{A}}$ when $L=0$. Therefore, our result holds for $L=0$

Assume $L=k+1$, let $\widehat{\mathcal{A}}$ denote the model generated by the first $k$ layers of $\mathcal{A}$. By induction, there is \emph{time-then-graph} model $\widehat{\mathcal{A}}'$ and function $f_{\widehat{\mathcal{A}}'}$ that captures output of $\widehat{\mathcal{A}}'$ on transformed collapsed graph. We can append a layer to $\widehat{\mathcal{A}}'$ to build $\mathcal{A}'$ that simulates $\mathcal{A}$. Let 
$\{C^{(L)},(A_j^{(L)})_{j=1}^{T|P_2|},A_{aux1}^{(L)},A_{\eaux}^{(L)},R^{(L)}\}$ denote the building blocks of layer $L$ of $\mathcal{A}$, and let $u^{*},g^{*},r^{*}$ denote functions used in the layer that will be appended to $\widehat{\mathcal{A}}'$. They are defined below:
\begin{equation}
g^*(\mulsetl(\mathbb{H}_i^{V,(l-1)},\mathbb{H}_j^{V,(l-1)},\mathbb{H}_{ij}^E|j\in\mathcal{N}(i))\mulsetr):=A_{aux1}^{(L)}(\mulsetl f_{\widehat{\mathcal{A}}'}(\mathbb{H}_{i}^{V,(l-1)},\mathbb{H}_{j}^{V,(l-1)},\mathbb{H}_{ij}^E)|j\in\mathcal{N}(i)\mulsetr)
\end{equation}
\begin{equation}\label{rott}
r^*(\mulsetl\mathbb{H}_{j}^{V,(l-1)}|j\in V(G)\mulsetr)=R^{(L)}\Bigl(\mulsetl\mathbb{H}_{j,0}^{V,(l-1)}|j\in V(G)\mulsetr\cup(\bigcup_{j\in V(G)}Dec( \mathbb{H}_{j,1}^{V,(l-1)}))\Bigl)
\end{equation}
\begin{equation}
u^*(\feature_g,\feature_r)=C^{(L)}(\feature_g,\feature_r)
\end{equation}
where $\feature_g,\feature_r$ are outputs of $g^*$ and $r^*$. $Dec(\fset)$ is a decoder function that do inverse mapping of $Enc(\fset)$ mentioned above, so $Dec( \mathbb{H}_{j,1}^{V,(l-1)})$ is actually $\mulsetl\feature_{e_{aj}}|j\in \mathcal{N}(a)\mulsetr$. Note that 
primal nodes in transformed graph only has type \iaux-  neighbors, so two inputs $\feature_g,\feature_r$, one for \iaux\  aggregation output 
 and one for global readout are already enough for computing the value. Comparing the three rules above with \Cref{acrgnn}, we can see that our new model $\mathcal{A}'$ perfectly captures $\mathcal{A}$.

We've captured $\mathcal{A}$, and the remaining work is to construct $f_{\mathcal{A}}$ defined above to complete inductive assumption. We can just choose a function that simulates message passing between pairs of added nodes $e_{ab}$ and $e_{ba}$ as well as message passing between $e_{ab}$ and $a$, and that function satisfies the condition for $f_{\mathcal{A}}$. Formally speaking, $f_{\mathcal{A}}$ can be defined below:

\begin{equation}
f_{\mathcal{A}}(\mathbb{H}_i^{V,(l)},\mathbb{H}_j^{V,(l)},\mathbb{H}_{ij}^{E}):=\textbf{Sim}_{\mathcal{A}_L}(\mathbb{H}_i^{V,(l-1)}, \mathbb{H}_g^{(l-1)},g_{ij},g_{ji}, \mathbb{H}_{ij}^{E})
\end{equation}
\begin{equation}
g_{ij}:=f_{\widehat{\mathcal{A}}'}(\mathbb{H}_i^{V,(l-1)},\mathbb{H}_j^{V,(l-1)},\mathbb{H}_{ij}^{E}), \mathbb{H}_g^{(l-1)}:=\mulsetl\mathbb{H}_i^{V,(l-1)}|i\in V(G)\mulsetr
\end{equation}
Let's explain this equation, $\textbf{Sim}_{\mathcal{A}_L}(a,g,s,b,e)$ is a local simulation function which simulates single-iteration message passing in the following scenario:

Suppose there is a graph $H$ with three constants $V(H)=\{a,e_{ab},e_{ba}\}$. There is an \iaux\  edge between $a$ and $e_{ab}$, an \eaux\  edge between $e_{ab}$ and $e_{ba}$, and additional edges of different types between $e_{ab}$ and $e_{ba}$. The description of additional edges can be founded in $e$. Initial node features of $a,e_{ab},e_{ba}$ are set to $a,s,b$ respectively. and the global readout output is $g$. Finally, run $L$-th layer of $\mathcal{A}$ on $H$, and $\textbf{Sim}_{\mathcal{A}_L}$ is node feature of $e_{ab}$ outputted by $\mathcal{A}_L$.

Note that if we use appropriate injective encoding or just use concatenation technic, $\mathbb{H}_g^{(l-1)}, \mathbb{H}_i^{V,(l-1)},\mathbb{H}_j^{V,(l-1)}$ can be accessed from $\mathbb{H}_i^{V,(l)},\mathbb{H}_i^{V,(l)}$. Therefore the above definition for $f_{\mathcal{A}}$ is well-defined. Moreover, in the above explanation we can see that $f_{\mathcal{A}}(\mathbb{H}_i^{V,(l-1)},\mathbb{H}_j^{V,(l-1)},\mathbb{H}_{ij}^{E})$ is exactly node feature of $e_{ij}$ outputted by $\mathcal{A}$ on the transformed collapsed graph, so our proof finishes.
\section{Proof of \Cref{ttransformation}}
\begin{reptheorem}{ttransformation}
R$^2$-TGNN $\circ F^T=$  R$^2$-TGNN $\circ F \circ H$.
\end{reptheorem}

First, we recall the definition for R$^2$-TGNN as in \Cref{acrtgnn}:
\begin{equation}\label{acrtgnn}
\textbf{x}_v^t=\mathcal{A}_t\biggl(G_t,v,\textbf{y}^{t}\biggl) \text{~~~~~~where~~~~~}\textbf{y}_v^{t}=[ I_{G_t}(v):\textbf{x}_v^{t-1}], \forall v\in V(G_t) 
\end{equation}

We say a R$^2$-TGNN is \emph{homogeneous} if $\mathcal{A}_1, \dots, \mathcal{A}_T$ share the same parameters. In particular, we first prove \Cref{homo}, namely, \emph{homogeneous} R$^2$-TGNN and R$^2$-TGNN~(where paramters in $\mathcal{A}_1, \dots, \mathcal{A}_T$ may differ) have the same expressiveness.

\begin{lemma}\label{homo}
homogenous R$^2$-TGNN $=$ R$^2$-TGNN
\end{lemma}
\begin{proof}
The forward direction \emph{homogeneous} R$^2$-TGNN$\subseteq$ R$^2$-TGNN trivially holds. It suffices to prove the backward direction.

Let $\mathcal{A}:\{\mathcal{A}_t\}_{t=1}^T$ denote a R$^2$-TGNN. Without loss of generality, we can assume all models in each timestamps have the same layer number $L$. Then for each $1\leq t\leq T$, we can assume all $\mathcal{A}_t$ can be represented by $\{C_t^{(l)},(A_{t,j}^{(l)})_{j=1}^{|P_2|},R_t^{(l)}\}_{l=1}^L$. Futhormore, without loss of generality, we can assume all output dimensions for $A_{t,j}^{(l)},R_t^{(l)}$ and $C_t^{(l)}$ are $d$. As for input dimension, all of these functions also have input dimension $d$ for $2\leq l\leq L$. Specially, by updating rules of R$^2$-TGNN \Cref{acrtgnn}, in the initialization stage of each timestamp we have to concat a feature with length $|P_1|$ to output of the former timestamp, so the input dimension for $A_{t,j}^{(1)},R_t^{(1)},C_t^{(1)}$ is $d+|P_1|$.

We can construct an equivalent \emph{homogeneous} R$^2$-TGNN with $L$ layers represented by $\{C^{*,(l)},(A_j^{*,(l)})_{j=1}^{|P_2|},R^{*,(l)}\}_{l=1}^L$. For $2\leq l\leq L$, $C^{*,(l)}A_j^{*,(l)},R^{*,(l)}$ use output and input feature dimension $d'=Td$. Similar to the discussion about feature dimension above, since we need to concat the unary predicates information before each timestamp, for layer $l=1$, $C^{*,(1)},A_j^{*,(1)},R^{*,(1)}$ have input dimension $d'+|P_1|$ and output dimension $d'$. For dimension alignment, $\feature_v^{0}$ used in \Cref{acrtgnn} is defined as zero-vector with length $d'$.

Next let's define some symbols for notation simplicity. For a feature vector $\feature$, let $\feature[i,j]$ denotes the slice of $\feature$ in dimension $[i,j]$. By the discussion above, in the following construction process we will only need feature $\feature$ with dimension $d'$ or $d'+|P_1|$. When $\feature$ has dimension $d'$, $\feature^{(i)}$ denotes $\feature[(i-1)d+1,id]$, otherwise it denotes $\feature[|P_1|+(i-1)d+1,|P_1|+id]$ . Let $[\feature_1......\feature_T]$ or $[\feature_t]_{t=1}^T$ denotes the concatenation of a sequence of feature $\feature_1......\feature_T$,  and $[\feature]^n$ denote concatenation of $n$ copies of $\feature$, $\textbf{0}^n$ denotes zero vectors of length $n$. Furthermore. Let $\fset$ denotes a multiset of $\feature$. Follows the updating rules defined in \Cref{acrgnn}, for all 
$1\leq j\leq |P_2|,1\leq l\leq L, A_j^{*,(l)},R^{*,(l)}$ should get input of form $\fset$, and the combination function $C^{*,(l)}$ should get input of form $(\feature_0,(\feature_j)_{j=1}^{|P_2|},\feature_g)$, where $\feature_0$ is from the node itself, $(\feature_j)_{j=1}^{|P_2|}$ are from aggregation functions $(A_j^{*,(l)})_{j=1}^{|P_2|}$ and $\feature_g$ is from the global readout $R^{*,(l)}$. The dimension of $\feature$ or $\fset$ should match the input dimension of corresponding function.  For all $1\leq l\leq L$, parameters in layer $l$ for the new model are defined below

\begin{equation}
l=1: C^{*,(l)}(\feature_0,(\feature_j)_{j=1}^{|P_2|},\feature_g):=[C_t^{(l)}([\feature_0[1,|P_1|],\feature_0^{(t-1)}],(\feature_j^{(t)})_{j=1}^{|P_2|},\feature_g^{(t)})]_{t=1}^T
\end{equation}
\begin{equation}
2\leq l\leq L: C^{*,(l)}(\feature_0,(\feature_j)_{j=1}^{|P_2|},\feature_g):=[C_t^{(l)}(\feature_0^{(t)},(\feature_j^{(t)})_{j=1}^{|P_2|},\feature_g^{(t)})]_{t=1}^T
\end{equation}
\begin{equation}
\forall j\in[K],l=1:  A_j^{*,(l)}(\fset)=[A_{t,j}^{(l)}(\mulsetl[\feature[1,|P_1|],\feature^{(t-1)}]|\feature\in\fset\mulsetr)]_{t=1}^T
\end{equation}
\begin{equation}
l=1:  R^{*,(l)}(\fset)=[R_{t}^{(l)}(\mulsetl[\feature[1,|P_1|],\feature^{(t-1)}]|\feature\in\fset\mulsetr)]_{t=1}^T
\end{equation}
\begin{equation}
\forall j\in[K],2\leq l\leq L:  A_j^{*,(l)}(\fset)=[A_{t,j}^{(l)}(\mulsetl\feature^{(t)}|\feature\in\fset\mulsetr)]_{t=1}^T
\end{equation}
\begin{equation}
2\leq l\leq L:  R^{*,(l)}(\fset)=[R_{t}^{(l)}(\mulsetl\feature^{(t)}|\feature\in\fset\mulsetr)]_{t=1}^T
\end{equation}
The core trick is to use 
$T$ disjoint slices $\feature^{(1)}......\feature^{(T)}$ to simulate $T$ different models $\mathcal{A}_1......\mathcal{A}_T$ at the same time, Since these slices are isolated from each other, a proper construction above can be found. The only speciality is that in layer $l=1$, we have to incorporate the unary predicate information $\feature[1,|P_1|]$ into each slice. By the construction above, we can see that for any node $v$, $\feature_v^{(T)}$ is exactly the its feature outputted by $\mathcal{A}$. Therefore, we finally construct an \emph{homogeneous} R$^2$-TGNN equivalent with $\mathcal{A}$. 
\end{proof}

% Hence, \emph{homogeneous} ACR-TGNN and ACR-TGNN are interchangeably used in this paper, and following previous works~\cite{jin2019recurrent,pareja2020evolvegcn,Park_2022,gao2022equivalence}, we choose \emph{homogeneous} ACR-TGNN in our experiments.

% In \Cref{acrtgnn}, we assume that $\mathcal{A}_t$ over timestamps $1\leq t\leq T$ has to be the same. Here we call it \emph{homogenous} ACR-TGNN following \cite{Barcelo2020The}. In this proof we have to distinguish \emph{homogenous} settings from normal ACR-TGNN. Therefore, although ACR-TGNN in the main body of our paper already implies the \emph{homogenous} property, here in this section when we say ACR-TGNN without \emph{homogenous} prefix, we mean the class of models which may have different $\mathcal{A}_t$ over timestamps.

% Therefore, \Cref{ttransformation} can be wrriten as follows:
% \begin{theorem}\label{ttransre}
% homogenous ACR-TGNN $\circ F^T=$ ACR-GNN $\circ F\circ H$
% \end{theorem}

% It seems that ACR-TGNN is more powerful than \emph{homogenous} ACR-TGNN. However, we can show actually they have the same expressiveness:
Now, we start to prove \Cref{ttransformation}.
\begin{reptheorem}{ttransformation}
R$^2$-TGNNs $\circ F^T=$ R$^2$-GNNs $\circ F\circ H$ on any universal graph class $\ggraph$. 
\end{reptheorem}

% Based on \Cref{homo}, we get rid of the restriction of parameter sharing. Therefore, in order to prove \Cref{ttransre}, we only need to prove the following theorem.

\begin{proof}
Since R$^2$-TGNN $\circ F^T$ only uses a part of predicates of $P'=F(H(P))$ in each timestamp, the forward direction R$^2$-TGNN $\circ F^T\subseteq$ R$^2$-GNN $\circ F\circ H$ trivially holds.

For any R$^2$-GNN $\mathcal{A}$ under $P'$, we want to construct an R$^2$-TGNN $\mathcal{A}'$ under $F^T(P)$ such that for any temporal knowledge graph $G$, $\mathcal{A}'$ outputs the same feature vectors as $\mathcal{A}$ on $F^T(G)$. We can assume $\mathcal{A}$ is represented as $(C^{(l)},(A_{j}^{(l)})_{j=1}^{K},A_{aux1}^{(l)},A_{aux2}^{(l)},R^{(l)})_{l=1}^L$, where $K=T|P_2|$.

First, by setting feature dimension to be $d'=T|P|+3$. We can construct an R$^2$-TGNN $\mathcal{A}'$ whose output feature stores all facts in $F(H(G))$ for any graph $G$. Formally speaking, $\mathcal{A}'$ should satisfy the following condition:

For any primal node $a$,  its feature outputted by $\mathcal{A}'\circ F^T$ should store all unary facts of form $A_i(a),A_i\in T|P_1|$ or $primal(a)$ on $F(H(G))$. For any non-primal node $e_{ab}$, its feature outputted by $\mathcal{A}'\circ F^T$ should store all binary facts of form $r_i(a,b),r_i\in T|P_2|$ or $r_{aux1}(a,b),r_{aux2}(a,b)$ where $b$ is another node on $F(H(G))$.

The $\mathcal{A}'$ is easy to construct since we have enough dimension size to store different predicates independently, and these facts are completely encoded into the initial features of corresponding timestamp. Let $(\mathcal{A}'_1......\mathcal{A}'_T)$ denote $\mathcal{A}'$.

Next, in order to simulate $\mathcal{A}$, we need to append some layers to $\mathcal{A}'_T$. Let $L$ denote the layer number of $\mathcal{A}$, we need to append $L$ layers represented as $(C^{*,(l)},(A_{j}^{*,(l)})_{j=1}^{|P_2|},A_{aux1}^{*,(l)},A_{aux2}^{*,(l)},R^{*,(l)})_{l=1}^L$

Since we have enough  information encoded in features, we can start to simulate $\mathcal{A}$. Since neighbor distribution of primal nodes don't change between $F^T(G)_T$ and $F(H(G))$,  
 it's easy to simulate all messages passed to primal nodes as destinations by $A_{aux1}^{*,(l)}$. For messages passed to non-primal node $e_{ab}$ as destination, it can be divided into messages from $a$ and messages from $e_{ba}$. The first class of messages is easy to simulate since the $aux1$ edge between $e_{ab}$ and $a$ is the same on $F^T(G)_T$ and $F(H(G))$.

For the second class of messages, since edges of type $r_i,1\leq i\leq T|P_2|$ may be lost in $F^T(G)_T$, we have to simulate these messages only by the unchanged edge of type $\textbf{aux2}$. It can be realized by following construction:
\begin{equation}\label{sim}
1\leq l\leq L,A_{aux2}^{*,(l)}(\fset)=[[A_j'^{,(l)}(\fset))]_{j=1}^{K},A_{aux2}^{(l)}(\fset)]
\end{equation}
where $K=T|P_2|, A_j'^{,(l)}(\fset):=A_j^{(l)}(\fset)$ if and only if $e_{ba}$ has neighbor $r_j$ on $F(H(G))$ , otherwise $A_j'^{(l)}(\fset):=\textbf{0}$. Note that $\fset$ is exactly the feature of $e_{ba}$, and we can access the information about its $r_j$ neighbors from feature since $\mathcal{A}'$ has stored information about these facts.

In conclusion, we've simulated all messages between neighbors. Furthermore, since node sets on $F^T(G)_T$ and $F(H(G))$ are the same, global readout $R^{(l)}$ is also easy to simulate by $R^{*,(l)}$. Finally, using the original combination function $C^{(l)}$, we can construct an R$^2$-TGNN on $F^T$ equivalent to $\mathcal{A}$ on $F(H(G))$ for any temporal knowledge graph $G$. 

\end{proof}
\section{Proof of \Cref{hierarchymain}}
\begin{figure}
    \centering
   
\begin{tikzpicture}[scale=0.80]
     
     \draw[draw=black] (-1,0.5) rectangle ++(12.5,-3);
		\node[fill=none,draw=none] (a) at (0,0) {R$^2$-TGNN};
		\node[fill=none,draw=none] (b) at (4,-2) {time-and-graph};
		\node[fill=none,draw=none] (c) at (4,0) {R$^2$-GNN$\circ H$};
		\node[fill=none,draw=none] (d) at (10,-0.3) {R$^2$-GNN$\circ F \circ H$ };
	    \node[fill=none,draw=none] (e) at (10,-0.9) { R$^2$-TGNN $\circ F^T$ };
	    \node[fill=none,draw=none] (f) at (10,-1.5) { time-then-graph};
	   	\draw[->] (a) to node[midway,above] {$\subsetneq$} (c);
	   	\draw[->] (c) to node[midway,above] {$\subsetneq$} (e);
	   	\draw[->] (b) to node[midway,below] {$\subsetneq$} (e);
     \draw[->] (a) to node[midway,above] {$\nsubseteq$} (b);
\end{tikzpicture}
\caption{Hierarchic expressiveness.\label{fig:hierarchy}}
\end{figure}
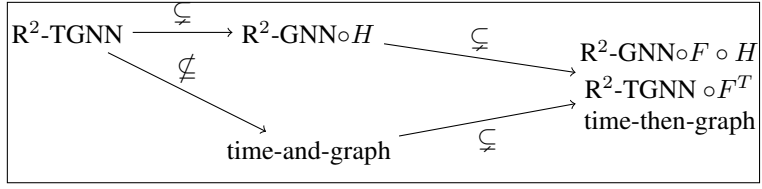
Based on \Cref{tgnnrelation}, \Cref{ttransformation} and \Cref{strictsep}, in order to prove \Cref{hierarchymain}, it suffices to show the following theorems.
\begin{theorem}\label{hier1}
If time range $T>1$ R$^2$-TGNN $\subsetneq$ R$^2$-GNN $\circ H$.
\end{theorem}
\begin{theorem}\label{hierr}
If time range $T>1$ R$^2$-TGNN $\nsubseteq$ time-and-graph.
\end{theorem}
\begin{proof}
Since a formal proof \Cref{hierr} relates to too many details in definition of time-and-graph (Please refer to \cite{gao2022equivalence}) which is not the focus here. We will just a brief proof sketch of \Cref{hierr}: That's because time-and-graph can not capture a chain of information that is continuously scattered in time intervals. Specifically, $\varphi(x)\coloneqq\exists^{\ge 1} y,\left(r^2_1(x,y)\wedge(\exists^{\ge 1} x,r^1_1(y,x))\right)$ can't be captured by time-and-graph but $\varphi(x)$ is in R$^2$-TGNN.

We mainly give a detaild proof of \Cref{hier1}: Since in each timestamp $t$, R$^2$-TGNN only uses 
 a part of predicates in temporalized predicate set $P'=H(P)$, 
  R$^2$-TGNN $\subseteq$ R$^2$-GNN $\circ H$ trivially holds. To show R$^2$-TGNN is strictly weaker than R$^2$-GNN $\circ H$. Consider the following classifier:

 Let time range $T=2$, and let $r$ be a binary predicate in $P_2$. Note that there are two different predicates $r^1,r^2$ in $P'=H(P)$. Consider the following temporal graph $G$ with $5$ nodes $\{1,2,3,4,5\}$. its two snapshots $G_1,G_2$ are as follows:

 $G_1=\{r(1,2),r(4,5)\}$

 $G_2=\{r(2,3)\}$.

 It follows that after transformation $H$, the static version of $G$ is:

 $H(G)=\{r_1(1,2),r_1(4,5),r_2(2,3)\}$.

 Consider the logical classifier $\exists y\Bigl(r_1(x,y)\wedge(\exists xr_2(x,y))\Bigl)$ under $P'$.It can be captured by some R$^2$-GNN under $P'$. Therefore,  R$^2$-GNN $\circ H$ can distinguish nodes $1,4$.

 However, any R$^2$-TGNN based on updating rules in \Cref{acrtgnn} can't distinguish these two nodes, so R$^2$-TGNN is strictly weaker than R$^2$-GNN $\circ H$.
\end{proof}

Based on \Cref{hier1}, we can consider logical classifier  $\mathbf{\varphi_3} \coloneqq \exists^{\ge 2}y(p_1^1(x,y)\wedge p_1^2(x,y))$. Note that this classifier is just renaming version of Figure~\ref{graphcompare}. Therefore $\varphi_3$ can't be captured by R$^2$-GNN $\circ H$, not to say weaker framework R$^2$-GNN by \Cref{hier1}.

.
\section{Experiment Supplementary}
\subsection{Synthetic dataset generation}
For each synthetic datasets, we generate 7000 graphs as tranining set and 500 graphs as test set. Each graph has $50-1000$ nodes. In graph generation, we fix the expected edge density $\delta$. In order to generate a graph with $n$ nodes, we pick $\delta n$ pairs of distinct nodes uniformly randomly. For each selected node pair $a,b$, each timestamp $t$ and each binary relation type $r$, we add $r^t(a,b)$ and $r^t(a,b)$ into the graph with independent probability $\frac{1}{2}$.

\begin{table}[t!]
     \centering
    \begin{tabular}{c|c|c|c|c}
        \hline
        datasets & $\varphi_1$ & $\varphi_2$ & $\varphi_3$ & $\varphi_4$ \\
        \hline
        Avg \# Nodes & 477 & 477 & 477 & 477 \\ 
        \hline
        Time\_range & 2 & 2 & 2 & 10 \\ 
        \hline
        \# Unary\ predicate &2 &2 &2 &3 \\
        \hline
        \# Binary\ predicate(non-temporalized) &1 &1 &1 &3 \\
        \hline 
        Avg \# Degree (in single timestamp) & 3 & 3 & 3 & 5 \\ 
        \hline 
        Avg \# positive percentage & 50.7 & 52 & 25.3 & 73.3 \\
        \hline
    \end{tabular}
    \caption{statistical information for synthetic datasets. \label{tab:syninfo}}
\end{table}

\begin{table}[t!]
     \centering
    \begin{tabular}{c|c|c|c}
        \hline
        datasets & AIFB & MUTAG & Brain-10 \\
        \hline
        \# Nodes & 8285 & 23644 & 5000 \\ 
        \hline
        Time\_Range & $ \backslash$ &  $\backslash$ & 12 \\ 
        \hline 
        \# Relation types & 45 & 23 & 20 \\ 
        \hline 
        \# Edges & 29043 & 74227 & 1761414 \\ 
        \hline 
        \# Classes & 4 & 2 & 10 \\ 
        \hline 
        \# Train Nodes & 140 & 272 & 4500 \\
        \hline
        \# Test Nodes & 36 & 68 & 500 \\
        \hline
    \end{tabular}
    \caption{statistical information for Real datasets.\label{tab:realinfo}}
\end{table}
\subsection{Statistical Information for Datasets}
We list the information for synthetic dataset in Table~\ref{tab:syninfo} and real-world dataset in Table~\ref{tab:realinfo}. Note that synthetic datasets contains many graphs, but real-world datasets only contains a single graph. Therefore, for real-world dataset, we have two disjoint node set as train split and test split for training and testing respectively. In training, the model can see the subgraph 
 induced by train split and unlabelled nodes, in testing, the model can see the whole graph but only evaluate the performance on test split.
\begin{table}[t!]
     \centering
    \begin{tabular}{c|c}
        \hline
        hyper-parameter & range \\
        \hline
        learning rate & 0.01 \\
        \hline
        combination & mean/max/add \\
        \hline
        aggregation/readout & mean/max/add \\
        \hline
        layer & $1,2,3$ \\
        \hline
        hidden dimension & $10,64,100$ \\
        \hline
    \end{tabular}
    \caption{Hyper-parameters.\label{tab:hyper}}
\end{table}
\subsection{Hyper-parameters}
For all experiments, we did grid search according to Table~\ref{tab:hyper}.
\subsection{More Results}
Apart from those presented in main part, we have some extra experimental results here:
\begin{table}[t!]
     \centering
    \begin{tabular}{c|c|c|c|c}
        \hline
        $\foc$ classifier & $\varphi_1$ & $\varphi_2$ & $\varphi_3$ & $\varphi_4$ \\
        \hline
        R-GAT $\circ H$ & 100 &61.4 &88.6 &82.0\\
        \hline
        R$^2$-GAT $\circ H$ & 100 &93.5 &95.0 &82.2\\ 
        \hline 
        R$^2$-GAT $\circ F\circ H$ & \textbf{100} &\textbf{98.2} &\textbf{100} &\textbf{95.8}\\
        \hline
    \end{tabular}
    \caption{Extra results on synthetic datasets\label{tab:ex_syn}}
\end{table}
\begin{table}[t!]
     \centering
    \begin{tabular}{c|c|c|c|c}
        \hline
         & AIFB & MUTAG & DGS & AM \\
         \hline
        \# of nodes  & 8285 &23644 &333845 &1666764\\
        \hline
        \# of edges  & 29043 &74227 &916199 &5988321\\
        \hline
        R-GCN  & 95.8 &73.2 &83.1 &89.3\\
        \hline
        R-GAT  & 96.9 &74.4 &86.9 &90.0\\
        \hline
        R-GNN  & 91.7 &76.5 &81.2 &89.5\\
        \hline
        R$^2$-GNN  & 91.7 &85.3 &85.5 &89.9\\ 
        \hline 
       R$^2$-GNN $\circ F$  & \textbf{97.2} &\textbf{88.2} &\textbf{88.0} &\textbf{91.4}\\ 
        \hline
    \end{tabular}
    \caption{Extra results for static real-world datasets.\label{tab:ex_realsta}}
\end{table}
\begin{table}[t!]
     \centering
    \begin{tabular}{c|c|c|c|c|c|c}
        \hline
         Models & GRU-GCN$\circ F^T$ & TGN $\circ F^T$ & R-TGNN & R-TGNN $\circ F^T$ & R$^2$-TGNN & R$^2$-TGNN$\circ F^T$ \\
         \hline
        Brain-10 & 95.0& 94.2&85.0&90.9&94.8&94.0\\
        \hline
    \end{tabular}
    \caption{Extra results for temporal real-world dataset Brain-10.\label{tab:ex_realtem}}
\end{table}

1. Extra results on synthetic datasets but using different base model architecture, where R-GAT refers to \cite{busbridge2019relational} and R$^2$-GAT refers to its extension with global readout. Please Refer to \Cref{tab:ex_syn}. These results show the generality of our results on different base models within the framework.

2. Extra results for static real-world datasets. Add a base model R-GAT\cite{busbridge2019relational} and two larger real-world datasets DGS and AM from \cite{schlichtkrull2018modeling}. Please refer to \Cref{tab:ex_realsta}. From the results for two bigger datasets DGM and AM, we can see our framework
  outperforms the other baselines, which confirms the scalability of our method and theoretical results.  These results show our method is effective both on small and large graphs.

3. Extra results for temporal real-world dataset Brain-10. Please refer to \Cref{tab:ex_realtem}. These results implies that our method is effective on different base models in temporal settings. Moreover, we can see separate improvements from global readout and graph transformation respectively. As we said in the main part, the drop in the last column may be due to the intrinsic drawbacks of current real-world datasets. Many real-world datasets can not be perfectly modeled as first-order-logic classifier.   This non-logical property may lead to less convincing experimental results. As \cite{Barcelo2020The}  commented,  these commonly used benchmarks are inadequate for testing advanced GNN variants.
% \subsection{more results}

%%%%%%%%%%%%%%%%%%%%%%%%%%%%%%%%%%%%%%%%%%%%%%%%%%%%%%%%%%%%

\end{document}